\documentclass[10pt,letterpaper]{article}
\usepackage[margin=1in]{geometry}
\usepackage{times}
\usepackage{graphicx}
\usepackage{xcolor}
\usepackage[authoryear,round]{natbib}
\setcitestyle{authoryear,round,citesep={;},aysep={,},yysep={;}}
\usepackage{microtype}

\usepackage{amsmath,amsfonts,bm}

\def\eqref#1{Eq.~(\ref{#1})}

\def\1{\bm{1}}

\DeclareMathAlphabet{\mathsfit}{\encodingdefault}{\sfdefault}{m}{sl}
\SetMathAlphabet{\mathsfit}{bold}{\encodingdefault}{\sfdefault}{bx}{n}

\newcommand{\E}{\mathbb{E}}

\usepackage{url}
\usepackage{booktabs}

\usepackage{amsmath}
\usepackage{amssymb}
\usepackage{amsthm}
\usepackage{algorithm}
\usepackage{algpseudocode}

\newtheorem{assumption}{Assumption}
\newtheorem{definition}{Definition}
\newtheorem{lemma}{Lemma}
\newtheorem{theorem}{Theorem}
\newtheorem{corollary}{Corollary}
\newtheorem{remark}{Remark}

\newcommand{\K}{\mathcal K}

\newcommand{\ip}[2]{\langle #1,#2\rangle}
\newcommand{\norm}[1]{\lVert #1\rVert}
\newcommand{\op}[1]{\lVert #1\rVert_{\mathrm{op}}}
\newcommand{\tr}{\operatorname{tr}}
\newcommand{\Int}{\operatorname{int}}

\title{Online Convex Optimization with Dueling Feedback}

\author{
Yiyang Lu \\
Purdue University 
\and
Hareshkumar Jadav \\
IIT Indore
\and
Mohammad Pedramfar \\
Mila - Quebec AI Institute/McGill University
\and
Ranveer Singh \\
IIT Indore
\and
Vaneet Aggarwal \\
Purdue University 
}
\date{}

\usepackage[hidelinks]{hyperref}
\hypersetup{pdftitle={Online Convex Optimization with Dueling Feedback},
  pdfauthor={Yiyang Lu; Hareshkumar Jadav; Mohammad Pedramfar; Ranveer Singh; Vaneet Aggarwal}}

\begin{document}

\maketitle

\begin{abstract}
Noisy binary comparison between two candidates is a common interface between human and learning systems, especially in modern large language model (LLM) post-training alignment. 
We study online convex optimization with dueling (pairwise comparison) feedback, where the learner observes only a binary preference between two queried points. 
We consider adversarial sequences of convex losses and measure regret with the loss at both queried points, under a comparison link with a known nonzero slope at the origin.
We propose a simple reduction that converts dueling feedback into approximate gradients, enabling the use of standard first-order methods.  We show that regret guarantees transfer under this reduction, yielding $\mathcal O(T^{3/4})$ static and adaptive regret, and $\mathcal O(T^{3/4}\sqrt{1+P_T/D})$ dynamic regret with unknown comparator path length $P_T$.
For strongly convex losses, the static and adaptive bounds improve to $\widetilde{\mathcal O}(T^{2/3})$.
For smooth losses, we presents unified dueling ellipsoidal FTRL, and proves $\widetilde{\mathcal O}(T^{2/3})$ static regret, which improves to $\widetilde{\mathcal O}(\sqrt T)$ under additional strong convexity.
\end{abstract}

\section{Introduction}

Preference-based learning has become central to modern AI systems, where specifying an explicit reward is difficult, but obtaining binary \emph{dueling} feedback via pairwise comparisons is convenient. 
In preference-based reinforcement learning from human feedback (RLHF) and direct preference optimization (DPO) pipelines, supervision is often reduced to pairwise preferences between candidate outputs, commonly modeled using a Bradley-Terry (logistic) link \citep{bradley1952rank,christiano2017deep,ouyang2022training,rafailov2023direct,muldrew2024active,deng2025less}.
The similar paradigm also appears in robotics \citep{biyik2024active, sadigh2017active}, recommendation systems \citep{hofmann2013reusing}, and clinical treatment~\citep{sui2017correlational}. 

The theory of dueling feedback is well-developed for finite action spaces \citep{yue2012dueling, ailon2014reducing, bengs2021preference, saha2021adversarial} and contextual settings \citep{dudik2015contextual, saha2022efficient, saha2025efficient}. For continuous domains, \cite{yue2009interactively} bound regret defined on pairwise preference probabilities for a fixed concave utility, \citet{lu2022non} extend and track time-varying utilities through their per-round optimizers, and other works study stationary objectives under different structural and preference assumptions \citep{kumagai2017regret,saha2025dueling,sharma2026lipschitz}.
Appendix~\ref{sec:related_work} discuss related works, and Table~\ref{tab:dueling_prior} compares the setting, metric and result of prior works in dueling convex optimization.

Many applications require optimizing time-varying objectives in adversarial environments for dynamic comparators, including adaptive resource allocation \citep{deb2024think, xiong2025dopl}, dynamic portfolio optimization \citep{sutiene2024enhancing}, and interactive systems such as recommendation platforms or post-training human-in-the-loop alignment, where user preferences evolve, tasks vary, and preference optimization is increasingly run online\citep{xiong2024iterative}. 
The tuned decision is often a continuous parameter, such as reward-mixture weights, or a decoding parameter, and two nearby settings are compared.

Online convex optimization (OCO) provides a principled framework for such adversarial settings, evaluating these decisions against fixed or changing comparators \citep{zhang2018adaptive, garber2022new}. 
However, these methods rely on gradient or function value, therefore cannot be directly applied in preference-based settings where only binary comparisons are available. This raises a fundamental question:
\begin{center}
\emph{Can we transfer adversarial online convex optimization results under limited dueling feedback?}
\end{center}
We study regret measured directly on the underlying convex losses and against arbitrary comparator sequences in the adversarial environment. Our focus is a general reduction that makes these OCO guarantees accessible through dueling feedback. 
The difficulty is fundamentally informational. Consider three feedback models. In the two-point bandit setting, the learner observes noisy function values at two nearby points within the same round, enabling estimation of function differences and hence gradients. In the one-point bandit setting, only a single noisy function value is observed, making gradient estimation possible but significantly noisier. In contrast, under dueling feedback, the learner does not observe function values at all, but only a binary outcome obtained by applying a (possibly noisy) transfer function to the difference of function values at two queried points. Thus, the learner receives only a \emph{1-bit, nonlinear transformation} of the information available in the two-point setting. This severe information constraint limits the ability to estimate gradients and makes it unclear whether standard online convex optimization techniques can be applied in this setting. 
We propose \emph{Duel-To-First-Order (DTFO)}, a reduction framework that converts dueling feedback into approximate gradient information, enabling the use of standard first-order OCO algorithms. 
The key challenge is to control estimation errors so that the resulting updates yield regret guarantees on the underlying losses for different comparator benchmarks. This is more challenging than stochastic settings, since the loss functions change arbitrarily over time and estimation errors accumulate into regret. We summarize our contributions to the community and our technical novelties as follows.

\paragraph{Contributions.}
\begin{itemize}

\item \textbf{A unified loss-based interval regret transfer} (Section~\ref{sec:dtfo}). We use a generalized notion of \emph{expected interval regret} over any fixed time interval $I \subseteq [1, T]$ against an arbitrary sequence of comparators $U = \{u_t\}_{t \in I}$ in \eqref{eq:interval-regret}.
The DTFO wrapper and Theorem~\ref{thm:query_regret_transfer_interval} convert first-order surrogate regret to expected interval regret, enabling integration of OCO algorithms. One quadratic-surrogate construction covers both convex and strongly convex losses, with the linear surrogate under $\sigma=0$.

\item \textbf{Static, dynamic, and adaptive consequences} (Section~\ref{sec:dtfo-consequences}). Instantiate DTFO with first-order methods gives the first four rows of Table~\ref{tab:oco_clean_fixed}, including $\mathcal O(T^{3/4})$ static and adaptive regret and $\mathcal O(T^{3/4}\sqrt{1+P_T/D})$ dynamic regret with unknown $P_T$ for convex losses, and $\tilde{\mathcal O}(T^{2/3}\ell_T^{1/3})$ static and adaptive regret for strongly convex loss.

\item \textbf{Improved static regret under smoothness} (Section~\ref{sec:ellipsoidal}). Under additional link bound, we present the unified dueling ellipsoidal FTRL (Algorithm~\ref{alg:dueling_ellipsoidal}) and establish Theorem~\ref{thm:universal_ellipsoidal}, attaining $\mathcal O(T^{2/3}\ell_T^{1/3})$ static regret, and under strong convexity, $\mathcal O(\sqrt{T\ell_T})$(Corollaries~\ref{cor:ellipsoid_smooth} and~\ref{cor:ellipsoid_strong_smooth}).
\end{itemize}

\paragraph{Technical Novelty.}
\begin{itemize}
\item \textbf{Comparison-to-surrogate transfer:} 
Lemma~\ref{lem:gradient_estimator} separates the smoothed underlying-loss gradient from link distortion through the residual \(\rho'(\Delta)+\gamma\), giving explicit \(O(\delta)\) spherical bias. Retaining loss curvature in a shared quadratic surrogate (\eqref{eq:sc_spherical_surrogate}) yields one transfer from pathwise first-order regret to regret at both queried actions, for fixed intervals and comparator paths (Theorem~\ref{thm:query_regret_transfer_interval}).

\item  \textbf{Ellipsoidal bias controlled by exploration geometry:} 
The comparison identity cancels the inverse exploration matrix before bounding bias, giving \(O(\delta^2\operatorname{tr}(A_t^2))\) error under the stronger link condition (\eqref{eq:ell-section-bias}). A feasible comparison retains curvature without perturbing the comparator (Lemma~\ref{lem:ellipsoid-feasible-comparison}), allowing this error to contract with the exploration ellipsoid in the barrier-FTRL analysis (Theorem~\ref{thm:universal_ellipsoidal}).
\end{itemize}

\begin{table}[t]
\centering
\small
\begin{tabular}{@{}llll@{}}
\toprule
\textbf{Loss class} & \textbf{Benchmark} & \textbf{Regret bound} & \textbf{Result} \\
\midrule
Convex & Static / Adaptive
 & $\mathcal O(T^{3/4})$ & Theorem~\ref{thm:orgd_static} \\
Convex & Dynamic
 & $\mathcal O(T^{3/4}\sqrt{1+P_T/D})$ & Theorem~\ref{thm:dynamic_orgd} \\
Strongly convex & Static
 & $\mathcal O(T^{2/3}\ell_T^{1/3})$ & Theorem~\ref{thm:sc_spherical_static} \\
Strongly convex & Adaptive
 & $\mathcal O(T^{2/3}\ell_T^{2/3})$ & Theorem~\ref{thm:sc_spherical_weakly_adaptive} \\
Smooth convex & Static
 & $\mathcal O(T^{2/3}\ell_T^{1/3})$ & Corollary~\ref{cor:ellipsoid_smooth} \\
Strongly convex, smooth & Static
 & $\mathcal O(\sqrt{T\ell_T})$ & Corollary~\ref{cor:ellipsoid_strong_smooth} \\
\bottomrule
\end{tabular}
\caption{Dueling regret guarantees of this work, for fixed problem parameters.
Here $\ell_T=1+\log T$ and $P_T=\sum_{t=2}^T\|w_t^*-w_{t-1}^*\|_2$. The first four rows respectively use DTFO with OGD, Improved Ader, SC-OGD, and binary IFLH (Section~\ref{sec:dtfo-consequences}), and the last two use ellipsoidal FTRL(Section~\ref{sec:ellipsoidal}). Each row uses its stated algorithm and tuning, and the first row gives both benchmarks for one run. The dynamic algorithm does not require $P_T$. The smooth rows additionally assume $|\rho'''|\le M_3$ and access to a self-concordant barrier. All comparators lie in $\mathcal K$.
}
\label{tab:oco_clean_fixed}
\end{table}

\section{Preliminaries}
\label{sec:preliminaries}
\paragraph{Online Setup and Assumptions.} We consider online convex optimization over a convex compact domain $\mathcal{K} \subset \mathbb{R}^d$ with diameter $D$. 
An oblivious adversary fixes the losses $f_1,\ldots,f_T$ before the learner's randomization. 
Each loss is $G$-Lipschitz on $\mathcal K$ for a fixed $G>0$ and has a known strong-convexity lower bound $\sigma\ge0$, where $\sigma=0$ is general convexity.
Let $\mathbb{B}_d$ be the closed unit ball in $\mathbb{R}^d$, and assume $\mathbb B_d\subseteq\mathcal K$. For a given perturbation radius $\delta > 0$, the $\delta$-shrunken domain $\mathcal{K}_\delta = \{ x \in \mathcal{K} \mid x + \delta v \in \mathcal{K}, \; \forall v \in \mathbb{B}_d \}$ contains the points in $\mathcal K$ with distance at least $\delta$ from the boundary, and $\delta$ is small enough that $\mathcal{K}_\delta\neq\varnothing$.

\paragraph{Dueling Feedback.} The learner does not observe gradient or values directly. Instead, at each round $t$, it queries two points $w_t^+, w_t^- \in \mathcal{K}$ and receives a noisy binary preference feedback indicating which point has lower loss. Let $\mathcal F_{t-1}$ denote the information available just before the current query direction is drawn, including the center $w_t$ and any internal randomization used to choose it. Then,
\begin{definition}[Dueling Oracle]
\label{def:dueling_oracle}
At each step $t$, when queried with a pair of points $(w^+, w^-) \in \mathcal{K} \times \mathcal{K}$, the dueling oracle returns a noisy binary preference bit $o_t \in \{-1, +1\}$. $o_t=1$ if $w^+$ is preferred otherwise $o_t=-1$. The feedback is generated probabilistically according to an underlying \textbf{transfer function} $\rho : \mathbb{R} \to [-1, 1]$ such that:
$$ \mathbb{E}[o_t \mid \mathcal F_{t-1}, w^+, w^-] = \rho\bigl(f_t(w^+) - f_t(w^-)\bigr). $$
\end{definition}

The notion of preference indicates that reversing the query order should reverse the preference probabilities.
If $P$ is the probability of $o_t=1$ for $(w^+,w^-)$, it is $1-P$ for $(w^-,w^+)$, so the conditional means are $2P-1$ and $1-2P$.
Thus, the transfer function should be odd (i.e., $\rho(z)=-\rho(-z)$), which implies $\rho(0)=0$.
In addition, we impose a mild regularity condition on the transfer function.
\begin{assumption}[Transfer Function]
\label{ass:oracle_general}
The function $\rho$ has a strictly negative first derivative $\rho'(0)=-\gamma$, where $\gamma>0$, and bounded curvature $|\rho''(x)| \le M$ for all $x \in [-2G,2G]$.
\end{assumption}
\noindent We also call $\rho$ links. Prior work \cite{saha2025dueling} allows transfer functions that may be flat at the origin. In contrast, we assume $\rho'(0) =-\gamma<0$, which ensures that local comparisons provide a usable first-order signal. This condition is satisfied by Bradley-Terry (logistic), Thurstone-Mosteller (Gaussian), and links that are linear on the relevant loss-difference range, and enables the construction of gradient estimators with controlled variance.

\paragraph{Regret.} To evaluate the algorithm's performance, we hold the algorithm accountable for the exact points queried to the oracle. For a fixed contiguous interval $I\subseteq[1,T]$ and comparator sequence $U=\{u_t\}_{t\in I}$ fixed independently of the learner's randomization, define the \emph{expected interval regret} for the sequence of losses $\{ f_t \}$ by
\begin{equation}
\label{eq:interval-regret}
    \text{Regret}(U, I) = \mathbb{E}\left[ \sum_{t \in I} \frac{f_t(w^+_t) + f_t(w^-_t)}{2} \right] - \sum_{t \in I} f_t(u_t).
\end{equation}
This general formulation recovers three notions of adversarial regret studied in this work. \textbf{Static Regret} evaluates performance over the full horizon $I = [1, T]$ against the single best fixed decision $u_t = w^* \in \mathcal{K}$, \(\text{Regret}_S(T) = \max_{w^* \in \mathcal{K}} \text{Regret}(\{w^*\}_{t=1}^T, [1, T]). \)
\textbf{Dynamic Regret} evaluates performance over the full horizon $I = [1, T]$ against an arbitrary, changing sequence of comparators $W^* = \{w^*_1, \dots, w^*_T\}$, \(\text{Regret}_D(W^*, T) = \text{Regret}(W^*, [1, T]).\)
\textbf{Adaptive Regret} ensures the algorithm performs well over any contiguous sub-interval $I = [t_1, t_2] \subseteq [1, T]$ against the best fixed decision for that specific interval $u_t = w^*_I \in \mathcal{K}$, \(\text{Regret}_{A}(I) = \max_{w^*_I \in \mathcal{K}} \text{Regret}(\{w^*_I\}_{t \in I}, I).\)
\begin{remark}[Adaptive and strongly adaptive regret.]\label{rmk:adaptive-regret}
We take the maximum over intervals of the expected regret against each interval's best fixed comparator, with expectation before maximization.
Our bounds depend polynomially on the horizon $T$, whereas strongly adaptive bounds scale sublinearly with interval length up to logarithmic factors in $T$.
This horizon dependence comes from a single global exploration radius $\delta$ tuned to $T$, whose reciprocal controls surrogate gradient magnitudes.
Retuning the radius separately for each interval would not establish a simultaneous guarantee for one run.
\end{remark}

\section{DTFO Reduction Framework}
\label{sec:dtfo}

Our primary contribution, the Duel-To-First-Order (DTFO) wrapper, acts as a modular interface between any standard first-order online convex optimization algorithm $\mathcal{A}$ and the restrictive noisy 1-bit dueling oracle. 
It queries two perturbed points on a small sphere around the prediction of $\mathcal{A}$ and turns comparison bit into a surrogate gradient. A single transfer inequality then handles both cases, for fixed, moving, and interval comparators.
\ref{sec:dtfo-wrapper} specifies the wrapper, \ref{sec:dtfo-estimator} establishes its comparison identity and bias bound, and \ref{sec:dtfo-transfer} transfers surrogate regret to expected interval regret. A known $\sigma\ge0$ covers both convex and strongly convex losses, and Appendices~\ref{sec:technical-lemmas}-\ref{sec:proof_of_thm:query_regret_transfer_interval} provide the technical lemmas and full proofs.

\subsection{Dueling-to-first-order wrapper}
\label{sec:dtfo-wrapper}

Fix $0<\delta\le1/2$ and run a base algorithm $\mathcal A$ on
$\mathcal K_\delta$. At round $t$, its prediction $w_t$ determines
the query pair $w_t^\pm=w_t\pm\delta U_t$, where $U_t$ is uniform
on $\mathbb S^{d-1}$. From the comparison bit $o_t$, DTFO constructs
\begin{equation}
 \widehat g_t=-\frac{d}{2\gamma\delta}o_tU_t,
 \qquad
 q_t(w)=\langle\widehat g_t,w\rangle
        +\frac{\sigma}{2}\|w-w_t\|_2^2.
 \label{eq:sc_spherical_surrogate}
\end{equation}
The surrogate loss $q_t(w)$ is linear when $\sigma=0$ and preserves curvature when $\sigma>0$. Algorithm~\ref{alg:duel_to_fo_wrapper} gives the wrapper; Appendix~\ref{sec:proofs_dtfo_consequences} supplies concrete base algorithms.

\begin{algorithm}[ht]
\caption{Duel-To-First-Order (DTFO) wrapper $\mathcal W(\mathcal A)$}
\label{alg:duel_to_fo_wrapper}
\begin{algorithmic}[1]
\Require Base algorithm $\mathcal A$, domain $\mathcal K$, radius
  $0<\delta\le1/2$, known $\gamma>0$ and $\sigma\ge0$.
\State Initialize $\mathcal A$ on $\mathcal K_\delta$.
\For{$t=1,\ldots,T$}
  \State Receive $w_t\in\mathcal K_\delta$ from $\mathcal A$.
  \State Draw $U_t\sim\operatorname{Unif}(\mathbb S^{d-1})$;
    query $w_t\pm\delta U_t$ and observe $o_t\in \{+1, -1\}$.
  \State Form $\widehat g_t$ and $q_t$ using
    \eqref{eq:sc_spherical_surrogate}.
  \State Update $\mathcal A$ with $q_t$, using
    $\nabla q_t(w)=\widehat g_t+\sigma(w-w_t)$ as needed.
\EndFor
\end{algorithmic}
\end{algorithm}

\paragraph{One shared curvature-preserving surrogate.}
If $\mathcal A$ aggregates experts, an expert predicting $w_{t,j}$ uses $\nabla q_t(w_{t,j})=\widehat g_t+\sigma(w_{t,j}-w_t)$, so every expert is updated from the same comparison, which enables first-order meta-learning algorithms with one shared gradient per round. On $\mathcal K_\delta$ the surrogate gradients are bounded by $H_\delta=d/(2\gamma\delta)+\sigma D$, and $\nabla q_t(w_t)=\widehat g_t$. 
Additionally, for every realized bit, the surrogate retains the original curvature of the quadratic surrogate loss through
\(
 q_t(w_t)-q_t(v_t)
 =\langle\widehat g_t,w_t-v_t\rangle
   -\frac\sigma2\|w_t-v_t\|_2^2,
\)
which permits strongly convex base algorithms to exploit the same curvature that appears in the original losses. Moreover, it $\sigma/H_\delta^2$-exp-concave on $\mathcal K_\delta$, because its Hessian is the known matrix $\sigma I_d$.

\subsection{Comparison identity, estimator bounds and bias control}
\label{sec:dtfo-estimator}

Lemma~\ref{lem:gradient_estimator} gives a comparison identity for spherical and ellipsoidal exploration, and then the norm and bias bounds for the spherical estimator. Appendix~\ref{sec:proof_of_lem:gradient_estimator} proves it in detail using the smoothing lemmas in Appendix~\ref{sec:technical-lemmas}.

\begin{lemma}[Comparison identity and spherical estimator bounds]
\label{lem:gradient_estimator}
Under Assumption~\ref{ass:oracle_general}, let $f$ be a convex
$G$-Lipschitz loss on $\mathcal K$.
Fix $x$ and a symmetric positive-definite matrix $A$ such that
$x+\delta A\mathbb B_d\subseteq\mathcal K$ and
$\delta\|A\|_{\mathrm{op}}\le1$. For
$U\sim\operatorname{Unif}(\mathbb S^{d-1})$ and
$V\sim\operatorname{Unif}(\mathbb B_d)$, define
\[
 \Delta(v)=f(x+\delta Av)-f(x-\delta Av),\qquad
 m(v)=\frac{\nabla f(x+\delta Av)+\nabla f(x-\delta Av)}2,
 \quad \bar g=\mathbb E_Vm(V).
\]
If $o$ is the comparison of
$x\pm\delta AU$ and $\widehat g=-d\,oA^{-1}U/(2\gamma\delta)$, then
\begin{equation}
 \mathbb E\widehat g=\bar g+B,
 \qquad
 B=-\frac1\gamma\mathbb E_V
       [\{\rho'(\Delta(V))+\gamma\}m(V)].
 \label{eq:dtfo-comparison-identity}
\end{equation}
In particular, for $A=I_d$,
\begin{equation}
 \|\widehat g\|_2=G_{\mathrm{surr}}:=\frac{d}{2\gamma\delta},
 \qquad
 \|B\|_2\le\frac{2d}{d+1}\frac{MG^2}{\gamma}\delta.
 \label{eq:dtfo-spherical-bias}
\end{equation}
\end{lemma}

\begin{proof}[Proof outline]
Applying the weak sphere-to-ball identity $d\mathbb E_U[h(U)U]=\mathbb E_V\nabla h(V)$ from Lemma~\ref{lem:gradient_extraction} to $h=\rho\circ\Delta$, with $\nabla\Delta(v)=2\delta A m(v)$, gives
\[
 \mathbb E\widehat g
 =-\frac{A^{-1}}{2\gamma\delta}
      \mathbb E_V\nabla(\rho\circ\Delta)(V)
 =-\frac1\gamma\mathbb E_V[\rho'(\Delta(V))m(V)].
\]
Thus $A^{-1}$ cancels before any norm bound. For $A=I_d$, use $|\rho'(z)+\gamma|\le M|z|$, $|\Delta(v)|\le2G\delta\|v\|_2$, $\|m(v)\|_2\le G$, and $\mathbb E\|V\|_2=d/(d+1)$. The estimator norm follows from $o^2=1$. 
\end{proof}
\paragraph{Integrate first, expand the link later.}
The sphere-to-ball identity is classical for one-point estimators, but never applied under the link. We apply it to the composition $\rho\circ\Delta$ of the link with the loss difference, before the link is linearized. The residual in~\eqref{eq:dtfo-comparison-identity} then involves $\rho'(\Delta)+\gamma$ rather than $\rho(\Delta)+\gamma\Delta$, the matrix $A^{-1}$ in $\widehat g$ cancels against the factor $A$ produced by the chain rule, and no derivative of $f$ at a query point is needed. The direct route bounds $|\rho(\Delta)+\gamma\Delta|\le M\Delta^2/2$ inside $\widehat g$ and multiplies by the estimator scale $d/(2\gamma\delta)$. It gives only $\|B\|_2\le dMG^2\delta/\gamma$, a factor $(d+1)/2$ worse than~\eqref{eq:dtfo-spherical-bias}. For ellipsoidal exploration it also retains a factor $\|A^{-1}\|_{\mathrm{op}}$, which grows without bound as the ellipsoid flattens near $\partial\mathcal K$. Section~\ref{sec:ellipsoidal} relies on this cancellation.

\subsection{Unified interval regret transfer}
\label{sec:dtfo-transfer}

Let $\mathcal F_{t-1}$ be the $\sigma$-algebra generated by the learner's history before drawing $U_t$, including past directions and comparisons and any internal randomization used to choose $w_t$.
Thus $w_t$ is $\mathcal F_{t-1}$-measurable, while $U_t$ is sampled uniformly independently of $\mathcal F_{t-1}$. 
Write $\mathbb E_t[\cdot]=\mathbb E[\cdot\mid\mathcal F_{t-1}]$ and let $\widetilde f_t(x)=\mathbb E_V f_t(x+\delta V)$ on $\mathcal K_\delta$. Define
\begin{equation}
 C_{\mathrm{sph}}=G(D+1)+\frac{2d}{d+1}\frac{MG^2D}{\gamma}.
 \label{eq:dtfo-transfer-constant}
\end{equation}
Theorem~\ref{thm:query_regret_transfer_interval} converts any pathwise surrogate-regret bound into an expected interval regret bound, with full proof in Appendix~\ref{sec:proof_of_thm:query_regret_transfer_interval}.

\begin{theorem}[DTFO interval regret transfer]
\label{thm:query_regret_transfer_interval}
Suppose the losses are $G$-Lipschitz and $\sigma$-strongly convex, where $\sigma\ge0$, and the feedback satisfies Assumption~\ref{ass:oracle_general}. For any fixed interval $I=[a,b]\subseteq[1,T]$ and comparator sequence $W_I^*=(w_t^*)_{t\in I}\subseteq\mathcal K$ independent of the learner's randomization, set $v_t=\Pi_{\mathcal K_\delta}(w_t^*)$. Algorithm~\ref{alg:duel_to_fo_wrapper} satisfies
\begin{equation}
 \mathbb E\!\left[\sum_{t\in I}
   \left(\frac{f_t(w_t^+)+f_t(w_t^-)}2-f_t(w_t^*)\right)\right]
 \le\mathbb E\!\left[\sum_{t\in I}
   \bigl(q_t(w_t)-q_t(v_t)\bigr)\right]
   +C_{\mathrm{sph}}\delta|I|.
 \label{eq:dtfo-transfer}
\end{equation}
Moreover, $\|v_t-w_t^*\|_2\le D\delta$ and
$\sum_{t=a+1}^b\|v_t-v_{t-1}\|_2
 \le\sum_{t=a+1}^b\|w_t^*-w_{t-1}^*\|_2$.
\end{theorem}

\begin{proof}[Proof outline]
Since $\mathbb B_d\subseteq\mathcal K$, $(1-\delta)w_t^*\in\mathcal K_\delta$. The projection error is therefore at most $\delta\|w_t^*\|_2\le D\delta$, and non-expansiveness of projection gives the path-length bound.
Let $\bar g_t=\mathbb E_V\nabla f_t(w_t+\delta V)$. Applying strong convexity between $w_t+\delta V$ and $v_t+\delta V$ and averaging yields
\begin{align*}
 \widetilde f_t(w_t)-\widetilde f_t(v_t)
 &\le\langle\bar g_t,w_t-v_t\rangle
       -\frac\sigma2\|w_t-v_t\|_2^2\\
 &\le\mathbb E_t[q_t(w_t)-q_t(v_t)]+D\|B_t\|_2,
\end{align*}
where Lemma~\ref{lem:gradient_estimator} gives
$\mathbb E_t\widehat g_t=\bar g_t+B_t$.
The averaged supporting inequality remains valid on
$\partial\mathcal K_\delta$; see Lemma~\ref{lem:smoothed_convexity}
in Appendix~\ref{sec:technical-lemmas}.

To account for both queried actions, couple a uniform ball point
as $V=RU$, where $R$ is independent of the uniform sphere point
$U$ and $\mathbb E R=d/(d+1)$. Symmetry and Lemma~\ref{lem:smoothing_approximation} give
\[
 \mathbb E_t\frac{f_t(w_t^+)+f_t(w_t^-)}2
       -\widetilde f_t(w_t)\le\frac{G\delta}{d+1},
 \qquad
 \widetilde f_t(v_t)-f_t(v_t)\le\frac{dG\delta}{d+1}.
\]
Add $f_t(v_t)-f_t(w_t^*)\le GD\delta$ and
the bias bound~\eqref{eq:dtfo-spherical-bias}, then sum over $I$
and take total expectations. All loss evaluations are in
$\mathcal K$. Appendix~\ref{sec:proof_of_thm:query_regret_transfer_interval}
gives the full proof.
\end{proof}
If $\mathcal A$ guarantees $\sum_{t\in I}[q_t(w_t)-q_t(v_t)]\le R_{\mathcal A}(v_{a:b},I)$ for every realized surrogate sequence, its expected query regret is at most $\mathbb E R_{\mathcal A}(v_{a:b},I)+C_{\mathrm{sph}}\delta|I|$. Choosing the comparator and the base bound gives static, dynamic, and adaptive guarantees.

\paragraph{Comparison-to-surrogate Reduction.} Reductions from bandit feedback to first-order methods usually rely on an unbiased estimate of the smoothed gradient and a single fixed comparator. However, our estimate is biased due to the link function, the surrogate gradient depends on the current bit, and the smoothed loss is centered at $w_t$ while the losses are charged at $w_t^\pm$. Theorem~\ref{thm:query_regret_transfer_interval} handles all three at once. It requires the bound on surrogate sequences, which accommodates the dependence of $q_t$ on the bit. It charges the link bias ($D\|B_t\|_2$), the smoothing of both queried points ($G\delta$, through the coupling $V=RU$), and the comparator projection ($GD\delta$) to one additive term $C_{\mathrm{sph}}\delta|I|$. The statement holds for any fixed interval and comparator path, and projection does not change the path length bound. Static, dynamic, and adaptive guarantees therefore follow from the base algorithm's analysis applied to $q_t$.

\section{Consequences of DTFO Wrapper: Algorithms and Regrets}
\label{sec:dtfo-consequences}

Using the Duel-To-FO (DTFO) wrapper as a meta-algorithm, we instantiate Theorem~\ref{thm:query_regret_transfer_interval} with standard first-order algorithms.
Section~\ref{sec:orgd} handles convex losses using Online Gradient Descent (OGD) and Improved Ader (IA), whereas Section~\ref{sec:improved_rates_strongly_convex} uses Strongly-Convex(SC)-OGD and Improved Following the Leading History (IFLH) for strongly convex losses.
Appendix~\ref{sec:proofs_dtfo_consequences} gives the implementation and proof details.
Let $0<\delta\le1/2$, $\ell_T=1+\log T$, and use the constants are as in Section~\ref{sec:dtfo}. Radius $\delta$ is chosen once per run for the whole horizon, and comparators belong to $\mathcal K$.

\subsection{Convex losses: static, interval, and dynamic regret}
\label{sec:orgd}
\label{sec:dynamic_orgd}

With $\sigma=0$, DTFO applied to projected OGD gives Online Relative Gradient Descent (ORGD). The usual projection argument~\citep{hazan2016introduction} applies pathwise to the shared linear surrogates. Theorem~\ref{thm:orgd_static} gives the static and adaptive regret of ORGD, with the implementation and proof in Appendix~\ref{sec:orgd_appendix}.

\begin{theorem}[Static and adaptive regret of ORGD]
\label{thm:orgd_static}
Under the assumptions of Theorem~\ref{thm:query_regret_transfer_interval},
use $\sigma=0$ and run Algorithm~\ref{alg:orgd} with $\eta_t=D/(G_{\mathrm{surr}}\sqrt t)$.
For every fixed interval $I=[s,e]\subseteq[1,T]$,
\begin{equation}
 \operatorname{Regret}_A(I)
 \le\frac{DG_{\mathrm{surr}}}{2}
       \left(\sqrt e+\sum_{t=s}^{e}t^{-1/2}\right)
       +C_{\mathrm{sph}}\delta|I|
 \le\frac32DG_{\mathrm{surr}}\sqrt T+C_{\mathrm{sph}}\delta|I|.
 \label{eq:consequences-ogd}
\end{equation}
A single run with $\delta=cT^{-1/4}$, $0<c\le1/2$, therefore has
$\operatorname{Regret}_S(T)=\mathcal O(T^{3/4})$ and
$\max_I\operatorname{Regret}_A(I)=\mathcal O(T^{3/4})$.
\end{theorem}

For moving comparators, Dynamic ORGD uses IA framework~\citep{zhang2018adaptive} , which calls OGD experts that share the comparison at the mixture point but with a geometric grid of step sizes.
Theorem~\ref{thm:dynamic_orgd} gives the resulting path-length bound, and Appendix~\ref{sec:dynamic_orgd_appendix} specifies the grid and aggregation rule and proves the guarantee without knowing $P_T$.

\begin{theorem}[Dynamic regret of Dynamic ORGD]
\label{thm:dynamic_orgd}
Under the assumptions of Theorem~\ref{thm:query_regret_transfer_interval}, set $\sigma=0$ and run Algorithm~\ref{alg:dynamic_orgd} as described. For every comparator path $W^*=(w_t^*)_{t=1}^T$ fixed independently of the learner's randomization, let $P_T=\sum_{t=2}^T\|w_t^*-w_{t-1}^*\|_2$. Then
\begin{equation}
 \operatorname{Regret}_D(W^*,T)
 \le\mathcal O\!\left(DG_{\mathrm{surr}}
                \sqrt{T(1+P_T/D)}\right)+C_{\mathrm{sph}}\delta T.
 \label{eq:consequences-ader}
\end{equation}
With $\delta=cT^{-1/4}$, $0<c\le1/2$, this gives $\mathcal O(T^{3/4}\sqrt{1+P_T/D})$ dynamic regret without knowing $P_T$, using one comparison and $\mathcal O(\log T)$ experts per round.
\end{theorem}

Similar to IA, Dynamic ORGD shares one duel per round at the mixture point across all $\mathcal O(\log T)$ experts and adapts to the unknown $P_T$. However, the radius must be common to all experts, because changing it changes the queried points and the comparison bit. This is why the radius is fixed at $cT^{-1/4}$ independently of $P_T$ (Appendix~\ref{sec:dynamic_orgd_appendix}). 


\subsection{Strongly convex losses: static and adaptive regret}
\label{sec:improved_rates_strongly_convex}

For $\sigma>0$, SC-OGD runs Algorithm~\ref{alg:orgd} with step size $\eta_t=1/(\sigma t)$~\citep{hazan2016introduction} on the quadratic surrogates in~\eqref{eq:sc_spherical_surrogate}. 
Theorem~\ref{thm:sc_spherical_static} gives its static regret, with full proof in Appendix~\ref{sec:sc-ogd-appendix}. 
\begin{theorem}[Static regret for strongly convex losses]
\label{thm:sc_spherical_static}
Under the assumptions of Theorem~\ref{thm:query_regret_transfer_interval} with $\sigma>0$, running Algorithm~\ref{alg:orgd} with step size $\eta_t=1/(\sigma t)$ satisfies
\begin{equation}
 \operatorname{Regret}_S(T)
 \le\frac{G_{\mathrm{surr}}^2}{2\sigma}\mathsf H_T
       +C_{\mathrm{sph}}\delta T,
 \qquad \mathsf H_T:=\sum_{t=1}^T\frac1t.
 \label{eq:consequences-sc-static}
\end{equation}
Choosing $\delta=c(\ell_T/T)^{1/3}$, $0<c\le1/2$, gives
$\mathcal O(T^{2/3}\ell_T^{1/3})$ static regret.
\end{theorem}

For intervals instead of whole horizon, binary IFLH~\citep{zhang2018dynamic} runs SC-OGD experts on the shared $q_t$. Their gradients are bounded by $H_\delta$ rather than $G_{\mathrm{surr}}$, and the static and interval guarantees use different algorithms and radii. Theorem~\ref{thm:sc_spherical_weakly_adaptive} gives the adaptive regret, and Appendix~\ref{sec:sc_iflh_appendix} gives the complete implementation and proof. 
\begin{theorem}[Adaptive regret for strongly convex losses]
\label{thm:sc_spherical_weakly_adaptive}
Under the assumptions of Theorem~\ref{thm:query_regret_transfer_interval} with $\sigma>0$,
Algorithm~\ref{alg:sc_iflh_dtfo} satisfies, for every fixed interval $I$,
\begin{equation}
 \operatorname{Regret}_A(I)
 \le\frac{H_\delta^2}{2\sigma}m_I(1+3\log T)
       +C_{\mathrm{sph}}\delta|I|,
 \qquad m_I=\lceil\log_2|I|\rceil+1.
 \label{eq:consequences-iflh}
\end{equation}
If we choose a single radius $\delta=c(\ell_T^2/T)^{1/3}$ with $0<c\le1/2$, the adaptive regret is $\max_I\operatorname{Regret}_A(I)=\mathcal O(T^{2/3}\ell_T^{2/3})$.
The algorithm uses one comparison query and $\mathcal O(\log T)$ active experts per round.
\end{theorem}
The SC-OGD cancellation and the exponential weighting of IFLH apply here because the curvature and exp-concavity of $q_t$ do not depend on the comparison bit (Section~\ref{sec:dtfo-wrapper}), and they require no smoothness of $f_t$. The logarithmic surrogate regret of SC-OGD therefore holds pathwise, and the rate improves from $T^{3/4}$ to $T^{2/3}\ell_T^{1/3}$ with no change to the dueling machinery.

\section{Smooth Losses and Ellipsoidal Estimators}
\label{sec:ellipsoidal}

Ellipsoidal exploration lets the iterates approach the boundary and exploits smoothness there, while a fixed Euclidean sphere would leave $\mathcal K$. 
Assume additionally that each $f_t$ has an $L$-Lipschitz gradient on $\mathcal K$ and that the odd link $\rho$ is three times continuously differentiable on $[-2G,2G]$, with $|\rho'''|\le M_3$.
We keep the known $\sigma\ge0$ and assume access to a $\nu$-self-concordant barrier $\Phi$ with $\nabla^2\Phi\succ0$ on $\operatorname{int}\mathcal K$ (Definition~\ref{def:barrier}), its Hessian, and the Follow-The-Regularized-Leader (FTRL) subroutine. 
Algorithm~\ref{alg:dueling_ellipsoidal} gives the Unified Dueling Ellipsoidal FTRL.

Lemma~\ref{lem:smooth-feas} ensures that the queries are feasibility, that $\|A_t\|_{\mathrm{op}}\le D$, and that $\|A_t\|_{\mathrm{op}}^2 \le 1/(\eta\sigma t)$ when $\sigma>0$. Lemma~\ref{lem:smooth-bias} gives local-norm identity \(\|\widehat g_t\|_{H_t^{-1}}^2=\frac{d^2}{4\gamma^2\delta^2}\), where $\|g\|_{H_t^{-1}}^2=g^\top H_t^{-1}g$. It also bounds the bias $B$ resulting from the comparison identity (\eqref{eq:dtfo-comparison-identity}) in Lemma~\ref{lem:gradient_estimator}.

\begin{algorithm}[ht]
\caption{Unified Dueling Ellipsoidal FTRL}
\label{alg:dueling_ellipsoidal}
\begin{algorithmic}[1]
\Require $0<\delta\le\delta_0:=\min\{1/2,1/D\}$,
 $0<\eta\le\gamma\delta/d$, known $\sigma\ge0$, $\gamma$, and barrier $\Phi$.
\State $w_1\gets\arg\min_{w\in\operatorname{int}\mathcal K}\Phi(w)$.
\For{$t=1,\ldots,T$}
 \State $H_t\gets\nabla^2\Phi(w_t)+\eta\sigma tI_d$;
        $A_t\gets H_t^{-1/2}$.
 \State Draw $U_t\sim\operatorname{Unif}(\mathbb S^{d-1})$
        independently of $\mathcal F_{t-1}$; query
        $w_t^\pm=w_t\pm\delta A_tU_t$ and observe $o_t$.
 \State $\widehat g_t\gets-d\,o_tA_t^{-1}U_t/(2\gamma\delta)$;
        $q_t(w)\gets\langle\widehat g_t,w\rangle
                      +\frac\sigma2\|w-w_t\|_2^2$.
 \State $w_{t+1}\gets\arg\min_{w\in\operatorname{int}\mathcal K}
        \{\Phi(w)+\eta\sum_{s=1}^tq_s(w)\}$.
\EndFor
\end{algorithmic}
\end{algorithm}

The \(H_t\) is the Hessian at \(w_t\) of the updated FTRL objective \(\Phi+\eta\sum_{s=1}^t q_s\), including the current surrogate’s curvature. It is available before observing the comparison because the unknown feedback enters \(q_t\) only through its linear term, while \(\nabla^2q_t=\sigma I_d\) is known. Sampling with it makes the local-norm identity above hold for the matrix that controls FTRL stability. Thus the curvature is known before the duel.

Section~\ref{sec:ellipsoidal-transfer} bounds the bias of the ellipsoidal gradient estimator, and gives the per-round comparison-to-surrogate transfer.
Section~\ref{sec:ellipsoidal-regret} uses Lemma~\ref{lem:ellipsoid-feasible-comparison} to establish static-regret bound before tuning for convex and strongly convex cases.
Appendix~\ref{sec:proofs_ellipsoidal} supplies the full details and proofs.

\subsection{Ellipsoidal gradient estimator, Quadratic bias and regret transfer}
\label{sec:ellipsoidal-transfer}

Condition on $\mathcal F_{t-1}$ and apply Lemma~\ref{lem:gradient_estimator} with $A=A_t$.
Oddness implies $\rho''(0)=0$, hence $|\rho'(z)+\gamma|\le M_3z^2/2$ on $[-2G,2G]$.
Using $|\Delta(v)|\le2G\delta\|A_tv\|_2$ and $\mathbb E\|A_tV\|_2^2=\operatorname{tr}(A_t^2)/(d+2)$ gives the quadratic bias bound proved for Lemma~\ref{lem:smooth-bias} in Appendix~\ref{sec:ellipsoid-bias-proof}:
\begin{equation}
 \|\mathbb E_t\widehat g_t-\bar g_t\|_2
 \le\frac{2M_3G^3\delta^2}{\gamma(d+2)}
       \operatorname{tr}(A_t^2),
 \qquad
 \bar g_t=\mathbb E_V\nabla f_t(w_t+\delta A_tV).
 \label{eq:ell-section-bias}
\end{equation}
The estimator contains $A_t^{-1}$, whose norm can grow as the exploration ellipsoid shrinks. Lemma~\ref{lem:gradient_estimator} cancels this matrix before bounding the conditional-mean error. Thus, \eqref{eq:ell-section-bias} scales with the squared semiaxis lengths of the exploration ellipsoid. Its dependence on $\delta^2\operatorname{tr}(A_t^2)$ matches the smoothness error of the queried losses, allowing both errors to decrease with the same geometry.
Lemma~\ref{lem:ellipsoid-feasible-comparison} is the per-round transfer for this geometry, from the queried loss to the surrogate loss. Its proof in Appendix~\ref{sec:proof-ellipsoid-feasible-comparison} retains curvature and allows boundary comparators without perturbing them. The transfer requires smoothness only on $\mathcal K$.
\begin{lemma}[Per-round Ellipsoidal Regret Transfer]
\label{lem:ellipsoid-feasible-comparison}
For every fixed $u\in\mathcal K$, Algorithm~\ref{alg:dueling_ellipsoidal}
satisfies
\begin{equation}
 \mathbb E_t\!\left[\frac{f_t(w_t^+)+f_t(w_t^-)}2-f_t(u)\right]
 \le\mathbb E_t[q_t(w_t)-q_t(u)]
       +c_{d,\sigma}\delta^2\operatorname{tr}(A_t^2),
 \label{eq:ell-section-comparison}
\end{equation}
where \(c_{d,\sigma}=\frac L{2d}+\frac{L-\sigma}{2(d+2)}
              +\frac{2M_3G^3D}{\gamma(d+2)}\).
Further, 
\(c_{d,\sigma}\operatorname{tr}(A_t^2)
    \le C\|A_t\|_{\mathrm{op}}^2\) 
where \(C:=L+\frac{2M_3G^3D}{\gamma}\).
\end{lemma}
\begin{proof}[Proof sketch]
Apply smoothness and strong convexity at the same feasible point
$y=w_t+\delta A_tV$, where $V$ is uniform on the unit ball.
Subtracting the inequalities and using $\mathbb E_V[V]=0$ gives
\[
 f_t(w_t)-f_t(u)
 \le \langle\bar g_t,w_t-u\rangle
       -\frac\sigma2\|w_t-u\|_2^2
       +\frac{L-\sigma}{2(d+2)}
          \delta^2\operatorname{tr}(A_t^2).
\]
This step evaluates the loss only at $y$, $w_t$, and $u$ and requires no feasibility of a perturbed comparator $u+\delta A_tV$. Adding the query-smoothing error and \eqref{eq:ell-section-bias} supplies the claim, with the retained negative quadratic term absorbed exactly into $q_t(w_t)-q_t(u)$. Appendix~\ref{sec:proof-ellipsoid-feasible-comparison} gives the full proof.
\end{proof}
The exploration ellipsoid lies in $\mathcal K$ around $w_t$, but it need not do so around a boundary comparator $u$, so these values can lie outside $\mathcal K$. Moving $u$ inward would cost a first-order term that does not scale with $\delta^2\operatorname{tr}(A_t^2)$. Lemma~\ref{lem:ellipsoid-feasible-comparison} instead applies smoothness and strong convexity at the same feasible point $y=w_t+\delta A_tV$ and compares directly with $f_t(u)$. The negative quadratic term survives exactly, every error is of order $\delta^2\operatorname{tr}(A_t^2)$, and smoothness is needed only on $\mathcal K$. Together these properties make the comparison compatible with quadratic-surrogate FTRL and its shrinking geometry.

\subsection{Improved Static Regret for Smooth Losses}
\label{sec:ellipsoidal-regret}

Theorem~\ref{thm:universal_ellipsoidal} combines the per-round query-to-surrogate transfer with the path-wise FTRL bound into a master static-regret guarantee. Appendix~\ref{sec:proof-ellipsoidal-master} gives the complete proof.

\begin{theorem}[Ellipsoidal FTRL static regret]
\label{thm:universal_ellipsoidal}
Under the assumptions of this section, for every $0<\alpha<1$,
Algorithm~\ref{alg:dueling_ellipsoidal} satisfies
\begin{equation}
 \operatorname{Regret}_S(T)
 \le\frac{\nu\log(1/\alpha)}\eta
       +\frac{\eta d^2T}{2\gamma^2\delta^2}
       +C\delta^2\sum_{t=1}^T\mathbb E\|A_t\|_{\mathrm{op}}^2
       +\alpha GDT.
 \label{eq:ell-section-master}
\end{equation}
\end{theorem}

\begin{proof}[Proof outline]
The barrier-FTRL bound applied to the quadratic surrogates (Lemma~\ref{lem:smooth-ftrl}) gives, for
every interior comparator $v$,
\[
 \sum_{t=1}^T[q_t(w_t)-q_t(v)]
 \le\frac{\Phi(v)-\Phi(w_1)}\eta
       +2\eta\sum_{t=1}^T\|\widehat g_t\|_{H_t^{-1}}^2.
\]
The step-size cap makes the Newton decrement at most $1/2$,
which controls the FTRL stability term~\citep{hazan2014bandit}.
For $u\in\mathcal K$, set $v=(1-\alpha)u+\alpha w_1$.
Barrier growth bounds $\Phi(v)-\Phi(w_1)$ by
$\nu\log(1/\alpha)$, and Lipschitzness costs at most
$\alpha GDT$ in the original losses. Combine
Lemma~\ref{lem:ellipsoid-feasible-comparison} with
\(\|\widehat g_t\|_{H_t^{-1}}^2=\frac{d^2}{4\gamma^2\delta^2}\) and take expectations.
\end{proof}
In \eqref{eq:ell-section-master}, the third term contains both nonlinear-link bias and smoothness error, and the other terms come from barrier FTRL and interior approximation of the comparator. At a fixed radius $\delta$, the estimator's local-norm cost is independent of $A_t$, while the geometric sum is bounded by $D^2T$ for convex losses and by $\mathsf H_T/(\eta\sigma)$ under strong convexity. The comparison-feedback error therefore fits the shrinking geometry of the FTRL analysis, leading to the two tunings below.

Set $\alpha=1/(T+1)$, $B_T=\nu\log(T+1)$, and $V_0=d^2/(2\gamma^2)$. Let $\eta$ be constant within each run. Corollary~\ref{cor:ellipsoid_smooth} uses $\|A_t\|_{\mathrm{op}}\le D$, and Appendix~\ref{sec:proof-ellipsoidal-convex} proves the tuning below, including the caps.
\begin{corollary}[Smooth convex tuning]
\label{cor:ellipsoid_smooth}
For $\sigma=0$ and $C_{\mathrm{geom}}:=CD^2>0$, choose \(\delta=\min\left\{\delta_0, \left(\frac{B_TV_0}{C_{\mathrm{geom}}^2T}\right)^{1/6}\right\},\) and \(\eta=\min\left\{\delta\sqrt{\frac{B_T}{V_0T}}, \frac{\gamma\delta}{d}\right\}.\)
For every $T$, regret is at most $B_T/\eta+V_0\eta T/\delta^2+C_{\mathrm{geom}}\delta^2T+GD$.
Once both caps are inactive, $\operatorname{Regret}_S(T)\le3(B_TV_0C_{\mathrm{geom}})^{1/3}T^{2/3}+GD=\mathcal O(T^{2/3}\ell_T^{1/3})$.
\end{corollary}
If $C_{\mathrm{geom}}=0$, a fixed feasible radius and the capped learning rate in Corollary~\ref{cor:ellipsoid_smooth} give $\mathcal O(\sqrt{T\ell_T})$ directly.
For $\sigma>0$, the bound $\|A_t\|_{\mathrm{op}}^2\le1/(\eta\sigma t)$ gives harmonic geometric error. Corollary~\ref{cor:ellipsoid_strong_smooth} uses this improvement, and Appendix~\ref{sec:proof-ellipsoidal-strong} proves its tuning.
\begin{corollary}[Strongly convex smooth tuning]
\label{cor:ellipsoid_strong_smooth}
For $\sigma>0$, fix $0<\delta\le\delta_0$ independent of $T$ and set \(V_\delta=V_0/\delta^2\), \(K_T=B_T+\frac{C\delta^2}{\sigma}\mathsf H_T,\) and \(\eta=\min\left\{\sqrt{\frac{K_T}{V_\delta T}}, \frac{\gamma\delta}{d}\right\}\).
For every $T$, regret is at most $K_T/\eta+V_\delta\eta T+GD$.
Once the learning-rate cap is inactive,
$\operatorname{Regret}_S(T)\le2\sqrt{V_\delta TK_T}+GD=\mathcal O(\sqrt{T\ell_T})$.
\end{corollary}

\paragraph{Harmonic accumulation of the dueling error.} Sampling with the regularized Hessian $\nabla^2\Phi(w_t)+\eta\sigma tI_d$ is the device of \citet{hazan2014bandit} for value feedback, where it controls the smoothing error. For comparisons bit, we have a second error, the link bias $B$ in \eqref{eq:dtfo-comparison-identity}. Without Lemma~\ref{lem:gradient_estimator} and \eqref{eq:ell-section-bias} that bias would carry $\|A_t^{-1}\|_{\mathrm{op}}$, which grows at least as $\sqrt{\eta\sigma t}$ because $H_t\succeq\eta\sigma tI_d$. Through Lemma~\ref{lem:gradient_estimator} and \eqref{eq:ell-section-bias} the link bias has the same $\delta^2\operatorname{tr}(A_t^2)$ form as the smoothness error, so both accumulate harmonically, and a radius fixed independently of $T$ yields $\mathcal O(\sqrt{T\ell_T})$.

\section{Conclusion}

We studied adversarial online convex optimization with one binary comparison per round and regret measured on the losses at both queried points.
The DTFO reduction uses controlled spherical gradient estimator and curvature-preserving quadratic surrogates, yielding a unified interval transfer for convex and strongly convex losses.
Its instantiations give static, path-length-dependent dynamic, and adaptive guarantees.
For smooth losses, a trace-controlled ellipsoidal bias bound and a feasible comparison argument lead to improved static regret, with strong convexity making the accumulated geometric error harmonic and hence further improvement.
Extending the framework to $K>2$ ranking feedback and more preference models to stay current with industrial practices are promising directions for future work.

\bibliography{main}

@inproceedings{abernethy2008competing,
  title={Competing in the dark: An efficient algorithm for bandit linear optimization},
  author={Abernethy, Jacob and Hazan, Elad and Rakhlin, Alexander},
  booktitle={Proceedings of the Annual Conference on Learning Theory},
  year={2008}
}

@inproceedings{agarwal2010optimal,
  title={Optimal algorithms for online convex optimization with multi-point bandit feedback.},
  author={Agarwal, Alekh and Dekel, Ofer and Xiao, Lin},
  booktitle={Proceedings of the Annual Conference on Learning Theory},
  year={2010}
}

@inproceedings{ailon2014reducing,
  title={Reducing dueling bandits to cardinal bandits},
  author={Ailon, Nir and Karnin, Zohar and Joachims, Thorsten},
  booktitle={Proceedings of the International Conference on Machine Learning},
  year={2014},
}

@article{bengs2021preference,
  title={Preference-based online learning with dueling bandits: A survey},
  author={Bengs, Viktor and Busa-Fekete, R{\'o}bert and El Mesaoudi-Paul, Adil and H{\"u}llermeier, Eyke},
  journal={Journal of Machine Learning Research},
  volume={22},
  number={1},
  year={2021}
}

@article{biyik2024active,
  author = {Bıyık, Erdem and Huynh, Nicolas and Kochenderfer, Mykel J. and Sadigh, Dorsa},
  title = {Active preference-based Gaussian process regression for reward learning and optimization},
  journal = {The International Journal of Robotics Research},
  volume = {43},
  number = {5},
  pages = {665--684},
  year = {2024}
}

@inproceedings{deb2024think,
  title={Think before you duel: Understanding complexities of preference learning under constrained resources},
  author={Deb, Rohan and Saha, Aadirupa and Banerjee, Arindam},
  booktitle={Proceedings of the International Conference on Artificial Intelligence and Statistics},
  year={2024}
}

@inproceedings{deng2025less,
  title={Less is more: Improving llm alignment via preference data selection},
  author={Deng, Xun and Zhong, Han and Ai, Rui and Feng, Fuli and Wang, Zheng and He, Xiangnan},
  booktitle={Proceedings of the Advances in Neural Information Processing Systems},
  year={2025}
}

@article{duchi2015optimal,
  title={Optimal rates for zero-order convex optimization: The power of two function evaluations},
  author={Duchi, John C. and Jordan, Michael I. and Wainwright, Martin J. and Wibisono, Andre},
  journal={IEEE Transactions on Information Theory},
  volume={61},
  number={5},
  pages={2788--2806},
  year={2015},
  publisher={IEEE}
}

@inproceedings{dudik2015contextual,
  title={Contextual dueling bandits},
  author={Dud{\'\i}k, Miroslav and Hofmann, Katja and Schapire, Robert E. and Slivkins, Aleksandrs and Zoghi, Masrour},
  booktitle={Proceedings of the Annual Conference on Learning Theory},
  year={2015},
}

@inproceedings{flaxman2005online,
  title={Online convex optimization in the bandit setting: gradient descent without a gradient},
  author={Flaxman, Abraham D. and Kalai, Adam Tauman and McMahan, H. Brendan},
  booktitle={Proceedings of the annual ACM-SIAM symposium on Discrete algorithms},
  year={2005},
}

@inproceedings{fokkema2024online,
  title={Online Newton method for bandit convex optimisation},
  author={Fokkema, Hidde and Van der Hoeven, Dirk and Lattimore, Tor and Mayo, Jack J.},
  booktitle={Proceedings of the Annual Conference on Learning Theory},
  year={2024},
}

@inproceedings{garber2022new,
  title={New projection-free algorithms for online convex optimization with adaptive regret guarantees},
  author={Garber, Dan and Kretzu, Ben},
  booktitle={Proceedings of the Annual Conference on Learning Theory},
  year={2022},
}

@article{hazan2016introduction,
  title={Introduction to online convex optimization},
  author={Hazan, Elad},
  journal={Foundations and Trends in Optimization},
  volume={2},
  number={3-4},
  pages={157--325},
  year={2016},
}

@article{hazan2007logarithmic,
  title={Logarithmic regret algorithms for online convex optimization},
  author={Hazan, Elad and Agarwal, Amit and Kale, Satyen},
  journal={Machine Learning},
  volume={69},
  pages={169--192},
  year={2007},
}

@inproceedings{hazan2014bandit,
  title={Bandit convex optimization: Towards tight bounds},
  author={Hazan, Elad and Levy, Kfir},
  booktitle={Proceedings of the Advances in Neural Information Processing Systems},
  year={2014}
}

@inproceedings{hofmann2013reusing,
  title={Reusing historical interaction data for faster online learning to rank for IR},
  author={Hofmann, Katja and Schuth, Anne and Whiteson, Shimon and De Rijke, Maarten},
  booktitle={Proceedings of the ACM international conference on Web search and data mining},
  pages={183--192},
  year={2013}
}

@inproceedings{kumagai2017regret,
  title={Regret analysis for continuous dueling bandit},
  author={Kumagai, Wataru},
  booktitle={Proceedings of the Advances in Neural Information Processing Systems},
  year={2017}
}

@inproceedings{muldrew2024active,
  title = {Active Preference Learning for Large Language Models},
  author = {Muldrew, William and Hayes, Peter and Zhang, Mingtian and Barber, David},
  booktitle = {Proceedings of the International Conference on Machine Learning},
  year = {2024},
}

@inproceedings{ouyang2022training,
  title={Training language models to follow instructions with human feedback},
  author={Ouyang, Long and Wu, Jeffrey and Jiang, Xu and Almeida, Diogo and Wainwright, Carroll L. and Mishkin, Pamela and others},
  booktitle={Proceedings of the Advances in Neural Information Processing Systems},
  year={2022}
}

@inproceedings{pedramfar2024linear,
  title={From linear to linearizable optimization: A novel framework with applications to stationary and non-stationary dr-submodular optimization},
  author={Pedramfar, Mohammad and Aggarwal, Vaneet},
  booktitle={Proceedings of the Advances in Neural Information Processing Systems},
  year={2024}
}

@inproceedings{rafailov2023direct,
  title={Direct preference optimization: Your language model is secretly a reward model},
  author={Rafailov, Rafael and Sharma, Archit and Mitchell, Eric and Ermon, Stefano and Manning, Christopher D and Finn, Chelsea},
  booktitle={Proceedings of the Advances in Neural Information Processing Systems},
  year={2023}
}

@inproceedings{sadigh2017active,
  title={Active preference-based learning of reward functions},
  author={Sadigh, Dorsa and Dragan, Anca D. and Sastry, Shankar and Seshia, Sanjit A.},
  booktitle={Robotics: Science and Systems},
  year={2017}
}

@inproceedings{saha2021adversarial,
  title={Adversarial dueling bandits},
  author={Saha, Aadirupa and Koren, Tomer and Mansour, Yishay},
  booktitle={Proceedings of the International Conference on Machine Learning},
  year={2021},
}

@inproceedings{saha2025dueling,
  title = {Dueling Convex Optimization with General Preferences},
  author = {Saha, Aadirupa and Koren, Tomer and Mansour, Yishay},
  booktitle = {Proceedings of the International Conference on Machine Learning},
  volume = {267},
  pages = {52552--52564},
  year = {2025},
}

@inproceedings{saha2022efficient,
  title={Efficient and optimal algorithms for contextual dueling bandits under realizability},
  author={Saha, Aadirupa and Krishnamurthy, Akshay},
  booktitle={Proceedings of the International Conference on Algorithmic Learning Theory},
  year={2022},
}

@inproceedings{saha2025efficient,
  title={Efficient and Near-Optimal Algorithm for Contextual Dueling Bandits with Offline Regression Oracles},
  author={Saha, Aadirupa and Schapire, Robert},
  booktitle={Proceedings of the Advances in Neural Information Processing Systems},
  year = {2025}
}

@inproceedings{saha2011improved,
  title={Improved regret guarantees for online smooth convex optimization with bandit feedback},
  author={Saha, Ankan and Tewari, Ambuj},
  booktitle={Proceedings of the International Conference on Artificial Intelligence and Statistics},
  year={2011},
}

@article{shamir2017optimal,
  title={An optimal algorithm for bandit and zero-order convex optimization with two-point feedback},
  author={Shamir, Ohad},
  journal={Journal of Machine Learning Research},
  volume={18},
  number={52},
  pages={1--11},
  year={2017}
}

@article{sharma2026lipschitz,
  title={Lipschitz Dueling Bandits over Continuous Action Spaces},
  author={Sharma, Mudit and Jain, Shweta and Aggarwal, Vaneet and Ghalme, Ganesh},
  journal={arXiv preprint arXiv:2604.00523},
  year={2026}
}

@inproceedings{sui2017correlational,
  title = {Correlational Dueling Bandits with Application to Clinical Treatment in Large Decision Spaces},
  author = {Sui, Yanan and Burdick, Joel W.},
  booktitle = {Proceedings of the International Joint Conference on Artificial Intelligence},
  year = {2017}
}

@article{sutiene2024enhancing,
  title={Enhancing portfolio management using artificial intelligence: literature review},
  author={Sutiene, Kristina and Schwendner, Peter and Sipos, Ciprian and Lorenzo, Luis and Mirchev, Miroslav and Lameski, Petre and Kabasinskas, Audrius and Tidjani, Chemseddine and Ozturkkal, Belma and Cerneviciene, Jurgita},
  journal={Frontiers in artificial intelligence},
  volume={7},
  year={2024},
}

@inproceedings{xiong2025dopl,
  title = {DOPL: Direct Online Preference Learning for Restless Bandits with Preference Feedback},
  author = {Xiong, Guojun and Dinesha, Ujwal and Mukherjee, Debajoy and Li, Jian and Shakkottai, Srinivas},
  booktitle = {Proceedings of the International Conference on Learning Representations},
  year = {2025}
}

@article{yue2012dueling,
  title={The K-armed dueling bandits problem},
  author={Yue, Yisong and Broder, Josef and Kleinberg, Robert and Joachims, Thorsten},
  journal={Journal of Computer and System Sciences},
  volume={78},
  number={5},
  pages={1538-1556},
  year={2012},
}

@inproceedings{zhang2018adaptive,
  title={Adaptive online learning in dynamic environments},
  author={Zhang, Lijun and Lu, Shiyin and Zhou, Zhi-Hua},
  booktitle={Proceedings of the Advances in Neural Information Processing Systems},
  year={2018}
}

@inproceedings{zhang2018dynamic,
  title     = {Dynamic Regret of Strongly Adaptive Methods},
  author    = {Zhang, Lijun and Yang, Tianbao and Jin, Rong and Zhou, Zhi-Hua},
  booktitle = {Proceedings of the International Conference on Machine Learning},
  year      = {2018}
}

@inproceedings{yue2009interactively,
  title={Interactively optimizing information retrieval systems as a dueling bandits problem},
  author={Yue, Yisong and Joachims, Thorsten},
  booktitle={Proceedings of the International Conference on Machine Learning},
  year={2009}
}

@inproceedings{lu2022non,
  title={Non-stationary dueling bandits for online learning to rank},
  author={Lu, Shiyin and Miao, Yuan and Yang, Ping and Hu, Yao and Zhang, Lijun},
  booktitle={Asia-Pacific Web (APWeb) and Web-Age Information Management (WAIM) Joint International Conference on Web and Big Data},
  year={2022},
}

@article{bradley1952rank,
  title={Rank analysis of incomplete block designs: {I}. The method of paired comparisons},
  author={Bradley, Ralph Allan and Terry, Milton E.},
  journal={Biometrika},
  volume={39},
  number={3/4},
  pages={324--345},
  year={1952}
}

@inproceedings{christiano2017deep,
  title={Deep reinforcement learning from human preferences},
  author={Christiano, Paul F. and Leike, Jan and Brown, Tom and Martic, Miljan and Legg, Shane and Amodei, Dario},
  booktitle={Proceedings of the Advances in Neural Information Processing Systems},
  year={2017}
}

@inproceedings{xiong2024iterative,
  title={Iterative Preference Learning from Human Feedback: Bridging Theory and Practice for {RLHF} under {KL}-constraint},
  author={Xiong, Wei and Dong, Hanze and Ye, Chenlu and Wang, Ziqi and Zhong, Han and Ji, Heng and Jiang, Nan and Zhang, Tong},
  booktitle={Proceedings of the International Conference on Machine Learning},
  volume={235},
  pages={54715--54754},
  year={2024}
}
\bibliographystyle{plainnat}

\newpage

\appendix
\section{Related Works}
\label{sec:related_work}

\paragraph{Preference-Based Learning and Dueling Bandits.}
Preference-based learning using pairwise comparisons was introduced through dueling bandits \citep{yue2012dueling}, with reductions to standard bandits in \cite{ailon2014reducing}. Extensive work studies stochastic and adversarial dueling bandits in discrete action spaces \citep{bengs2021preference, saha2021adversarial}, as well as contextual extensions \citep{dudik2015contextual, saha2022efficient, saha2025efficient}. These works study finite or contextual action spaces, and the continuous-action
extension of \citet{saha2025efficient} is discussed below.

\paragraph{Continuous Dueling for Convex Objectives.}
Several works extend dueling feedback to continuous domains. 
\cite{yue2009interactively} consider a compact convex parameter space with a fixed concave utility function and develop Dueling Bandit Gradient Descent (DBGD), obtaining an \(O(T^{3/4})\)-type regret guarantee measured through pairwise preference probabilities.
\citet{kumagai2017regret} studies a fixed strongly convex, smooth objective, with preference-regret guarantees called `Dueling-Bandit(DB)-regret' and a possible conversion to objective regret under link-slope bounds.
\cite{lu2022non} extend the DBGD line to non-stationary Online Learning to Rank (OL2R) for time-varying utilities, deriving a path-length-dependent \(O(T^{3/4})\)-type dynamic preference-regret bound and a meta-learning method that avoids prior knowledge of the path length.
More recent work by \cite{saha2025dueling} considers smooth convex objectives under general preference links in a stationary pure-exploration setting with sample complexity guarantees, while \cite{sharma2026lipschitz} studies stochastic Lipschitz dueling bandits over continuous action spaces. 
The continuous-action model of \citet{saha2025efficient} adopt a contextual approach and assumes linear objectives with respect to actions under stochastic feedback, using a Best-Response(BR)-Regret notation.

Table~\ref{tab:dueling_prior} summarizes these results. These works provide important precursors for continuous dueling optimization, but differ from the adversarial OCO formulation studied here. In \cite{yue2009interactively} and \cite{lu2022non}, cumulative regret is measured through pairwise preference probabilities induced by a link function, whereas we measure regret directly in the underlying convex objective values at the queried points. Moreover, our DTFO framework transfers guarantees from general first-order OCO algorithms and supports fixed, changing, and interval comparators within a unified objective-regret formulation.

\begin{table}[h]
\centering
\caption{Selected continuous dueling results. The metrics and comparator classes differ: $C_T$ is the path length of per-round optimizers, while $P_T(U)$ is the path length of an arbitrary comparator sequence. Rates suppress dimension, geometry, and link constants.}
\label{tab:dueling_prior}
\small
\begin{tabular}{@{}p{.22\textwidth}p{.42\textwidth}p{.30\textwidth}@{}}
\toprule
\textbf{Reference} & \textbf{Setting and benchmark} & \textbf{Guarantee} \\
\midrule
\cite{yue2009interactively}
& Fixed strictly concave utility; preference regret against its maximizer
& $\mathcal{O}(T^{3/4})$ \\
\addlinespace
\cite{lu2022non}
& Time-varying utilities in the DBGD model; preference regret against per-round maximizers
& $\tilde{\mathcal{O}}(T^{3/4}\sqrt{1+C_T})$\newline Unknown $C_T$ \\
\addlinespace
\cite{kumagai2017regret}
& Noisy strongly convex, smooth loss; preference regret with objective-value conversion
& $\mathcal{O}(\sqrt{T\log T})$ \\
\addlinespace
\cite{saha2025dueling}
& Fixed smooth convex loss; optimization error under general preferences
& Sample complexity depending on accuracy and transfer function \\
\addlinespace
This work (convex)
& Changing convex losses; objective-value regret against arbitrary comparators
& $\mathcal{O}(T^{3/4}\sqrt{1+P_T(U)})$\newline Unknown $P_T(U)$ \\
\bottomrule
\end{tabular}
\end{table}

\paragraph{Bandit Online Convex Optimization.}
Classical one-point~\citep{flaxman2005online} and two-point~\citep{agarwal2010optimal,duchi2015optimal,shamir2017optimal} methods construct gradient estimates from function evaluations. Our convex static bound has the same $T^{3/4}$ horizon dependence as the classical one-point value-feedback baseline~\citep{flaxman2005online}. Computationally efficient online Newton methods attain $\widetilde{\mathcal O}(\sqrt T)$ adversarial regret with scalar-value feedback~\citep{fokkema2024online}, using gradient and Hessian estimates of surrogate losses from possibly noisy observations.
These value-based estimators are not directly available from a binary comparison. Our contribution is a reduction from dueling feedback to first-order surrogate regret, giving static, dynamic, and adaptive objective-regret guarantees through standard OCO methods~\citep{zhang2018adaptive,garber2022new}. Related reduction frameworks transfer linear or quadratic regret guarantees to broader objectives and connect first-order and value feedback \citep{pedramfar2024linear}, whereas this work constructs and controls a surrogate gradient from binary comparisons.

In this paper, we study adversarial online convex optimization with general convex objectives under purely comparison-based feedback, and provide the regret guarantees measured directly on objective function value.

\section{Technical Lemmas}
\label{sec:technical-lemmas}

The three lemmas below establish curvature preservation, the weak sphere-to-ball
identity, and smoothing-error bounds, respectively. Appendices~\ref{sec:proof_of_lem:gradient_estimator}
and~\ref{sec:proof_of_thm:query_regret_transfer_interval} use these ingredients
to prove the estimator and interval-transfer results.

\paragraph{Convexity Preservation.} To build the theoretical bridge between the true objective functions and the base algorithm's surrogate losses, we rely on the properties of spherically smoothed functions. 
Throughout this appendix, $\mathcal K$ satisfies the domain assumptions, $0<\delta\le1/2$, and $U$ and $V$ are uniform on $\mathbb S^{d-1}$ and $\mathbb B_d$, respectively. 
First, we establish that the smoothing operation preserves the convexity and, when present, strong convexity of the original function. This is a standard property in bandit convex optimization \citep{flaxman2005online,hazan2016introduction}.

\begin{lemma}[Convexity of Smoothed Loss]
\label{lem:smoothed_convexity}
Let $f:\mathcal K\to\mathbb R$ be $G$-Lipschitz and
$\sigma$-strongly convex, with $\sigma\ge0$. Define
\[
 \widetilde f(w)=\mathbb E_V f(w+\delta V),
 \qquad w\in\mathcal K_\delta.
\]
Then $\widetilde f$ is $\sigma$-strongly convex on $\mathcal K_\delta$. Moreover, for any $x,z\in\mathcal K_\delta$,
\begin{equation}
 \widetilde f(x)-\widetilde f(z)
 \le\langle\bar g(x),x-z\rangle
       -\frac\sigma2\|x-z\|_2^2,
 \qquad
 \text{where} \ \bar g(x)=\mathbb E_V\nabla f(x+\delta V),
 \label{eq:dtfo-feasible-support}
\end{equation}
where gradients of $f$ are taken almost everywhere in
$\operatorname{int}\mathcal K$. Setting $\sigma=0$ gives convexity.
\end{lemma}

\begin{proof}
For $x,z\in\mathcal K_\delta$ and $\lambda\in[0,1]$, both
$x+\delta V$ and $z+\delta V$ are feasible, and
\[
 \lambda x+(1-\lambda)z+\delta V
 =\lambda(x+\delta V)+(1-\lambda)(z+\delta V).
\]
Applying strong convexity inside the expectation gives
\begin{align*}
 \widetilde f(\lambda x+(1-\lambda)z)
 &\le\mathbb E_V\left[
     \lambda f(x+\delta V)+(1-\lambda)f(z+\delta V)
     -\frac\sigma2\lambda(1-\lambda)\|x-z\|_2^2\right]\\
 &=\lambda\widetilde f(x)+(1-\lambda)\widetilde f(z)
     -\frac\sigma2\lambda(1-\lambda)\|x-z\|_2^2.
\end{align*}
This proves the first claim.

For the supporting inequality, note that $x+\delta v$ lies in
$\operatorname{int}\mathcal K$ whenever $\|v\|_2<1$, even when
$x\in\partial\mathcal K_\delta$. A convex $G$-Lipschitz function
is differentiable almost everywhere in the interior, with gradient
norm at most $G$. Thus, almost surely,
\[
 f(z+\delta V)\ge f(x+\delta V)
   +\langle\nabla f(x+\delta V),z-x\rangle
   +\frac\sigma2\|x-z\|_2^2.
\]
Taking expectations and rearranging proves
\eqref{eq:dtfo-feasible-support}. 
\end{proof}

\paragraph{Gradient Extraction Identity.} Next, we establish the gradient extraction identity.  We also state its weak sphere-to-ball form, which will be applied to the composed comparison function in the proof of Lemma~\ref{lem:gradient_estimator}. This is the integration identity underlying spherical smoothing \citep{flaxman2005online}.

\begin{lemma}[Gradient Extraction via Spherical Symmetry]
\label{lem:gradient_extraction}
For any Lipschitz scalar function $h$ on the closed unit ball,
\begin{equation}
 d\mathbb E_U[h(U)U]=\mathbb E_V\nabla h(V),
 \label{eq:dtfo-weak-gauss}
\end{equation}
where the right side uses the weak gradient. In particular, for a convex
$G$-Lipschitz function $f:\mathcal K\to\mathbb R$ and
$w\in\mathcal K_\delta$, set
$\Delta(U)=f(w+\delta U)-f(w-\delta U)$. Then
\begin{equation}
 \frac d{2\delta}\mathbb E_U[\Delta(U)U]
 =\bar g(w):=\mathbb E_V\nabla f(w+\delta V).
 \label{eq:dtfo-gradient-extraction}
\end{equation}
For $w\in\operatorname{int}\mathcal K_\delta$, this vector equals
$\nabla\widetilde f(w)$. Equation~\eqref{eq:dtfo-gradient-extraction}
also holds for boundary centers.
\end{lemma}

\begin{proof}
For a continuously differentiable $h$, apply the divergence theorem to
$h e_i$ for each coordinate $i$:
\[
 \int_{\mathbb B_d}\partial_i h(v)\,dv
 =\int_{\mathbb S^{d-1}}h(u)u_i\,dS(u).
\]
Since the surface area of $\mathbb S^{d-1}$ is $d$ times the volume
of $\mathbb B_d$, normalization gives~\eqref{eq:dtfo-weak-gauss}.
For Lipschitz $h$, extend $h$ to a Lipschitz function on
$\mathbb R^d$ and mollify. The mollifications converge uniformly on
the closed ball, and their gradients converge in $L^1(\mathbb B_d)$
to its weak gradient. Passing to the limit proves the same identity.
This auxiliary extension of $h$ does not require the original loss
$f$ to be defined outside $\mathcal K$.

For the second claim, expanding $\Delta$ gives
\begin{align}
 \frac d{2\delta}\mathbb E_U[\Delta(U)U]
 =\frac d{2\delta}\left(
   \mathbb E_U[f(w+\delta U)U]
   -\mathbb E_U[f(w-\delta U)U]\right).
 \label{eq:lemma_expansion}
\end{align}
By spherical symmetry, $U$ and $-U$ have the same distribution, so
\[
 \mathbb E_U[f(w-\delta U)U]
 =-\mathbb E_U[f(w+\delta U)U].
\]
Substituting into~\eqref{eq:lemma_expansion} and applying
\eqref{eq:dtfo-weak-gauss} to $h(v)=f(w+\delta v)$ yields
\[
 \frac d{2\delta}\mathbb E_U[\Delta(U)U]
 =\frac d\delta\mathbb E_U[f(w+\delta U)U]
 =\mathbb E_V\nabla f(w+\delta V)=\bar g(w).
\]
At an interior center, differentiation of the ball average is justified
by dominated difference quotients: $f$ is differentiable at
$w+\delta V$ almost surely, and its difference quotients in each fixed
direction are bounded by $G$ times the direction's norm.
Thus $\bar g(w)=\nabla\widetilde f(w)$ there. At boundary centers,
the averaged supporting inequality~\eqref{eq:dtfo-feasible-support}
provides the property needed for regret transfer.
\end{proof}

\paragraph{Controlled Smoothing Error.} Finally, we bound the approximation error introduced by the smoothing operation and the difference between the expected query loss and the ball average. Because the original function is $G$-Lipschitz, the smoothed function over a $\delta$-radius ball is guaranteed to be point-wise close to the true function, allowing us to bound the bias introduced by our wrapper.

\begin{lemma}[Approximation Error of Smoothed Loss]
\label{lem:smoothing_approximation}
Let $f:\mathcal K\to\mathbb R$ be $G$-Lipschitz and let
$\widetilde f(w)=\mathbb E_V f(w+\delta V)$ on
$\mathcal K_\delta$. For every $w\in\mathcal K_\delta$,
\begin{equation}
 |\widetilde f(w)-f(w)|\le\frac{d}{d+1}G\delta\le G\delta.
 \label{eq:dtfo-ball-approximation}
\end{equation}
Writing
$Q(w)=\mathbb E_U[f(w+\delta U)+f(w-\delta U)]/2$, we also have
\begin{equation}
 Q(w)-\widetilde f(w)\le\frac{G\delta}{d+1}.
 \label{eq:dtfo-sphere-ball-gap}
\end{equation}
\end{lemma}

\begin{proof}
By the triangle inequality and Lipschitzness,
\begin{align*}
 |\widetilde f(w)-f(w)|
 &\le\mathbb E_V|f(w+\delta V)-f(w)|\\
 &\le G\delta\,\mathbb E\|V\|_2
 =\frac{d}{d+1}G\delta.
\end{align*}
For the second claim, represent a uniform ball point as $V=RU$, where $R$ and $U$ are independent and $R$ has density $dr^{d-1}$ on $[0,1]$. Hence $\mathbb E R=d/(d+1)$. By spherical symmetry, $Q(w)=\mathbb E_U f(w+\delta U)$, and coupling the two expectations
gives
\begin{align*}
 Q(w)-\widetilde f(w)
 &=\mathbb E_{R,U}[f(w+\delta U)-f(w+\delta RU)]\\
 &\le G\delta\,\mathbb E(1-R)=\frac{G\delta}{d+1}.
\end{align*}
All points in these expectations lie in $\mathcal K$.
\end{proof}

\section{Proof of Lemma~\ref{lem:gradient_estimator}}
\label{sec:proof_of_lem:gradient_estimator}

We apply the weak sphere-to-ball identity from Appendix~\ref{sec:technical-lemmas}
to $\rho\circ\Delta$, cancel $A^{-1}$, and isolate the nonlinear bias.
Specializing to $A=I_d$ then gives the spherical estimator norm and linear bias bound.

\begin{proof}
Fix $x$ and $A$ as in the lemma. When they are chosen using the past,
condition on the information available before sampling the current
direction throughout this argument. Define
\begin{align*}
 \Delta(v)&=f(x+\delta Av)-f(x-\delta Av),\\
 m(v)&=\frac{\nabla f(x+\delta Av)+\nabla f(x-\delta Av)}2,
 \qquad \bar g=\mathbb E_Vm(V).
\end{align*}
Feasibility ensures that $x\pm\delta Av$ lie in $\operatorname{int}\mathcal K$ for $\|v\|_2<1$. Since $A$ is invertible, the affine maps preserve null sets, so both gradients exist almost everywhere and $\|m(v)\|_2\le G$. Lipschitzness also gives
\[
 |\Delta(v)|\le2G\delta\|Av\|_2\le2G,
\]
using $\delta\|A\|_{\mathrm{op}}\le1$.

By dueling oracle definition and the tower property,
\[
 \mathbb E\widehat g
 =-\frac d{2\gamma\delta}
     A^{-1}\mathbb E_U[\rho(\Delta(U))U].
\]
The function $h(v)=\rho(\Delta(v))$ is Lipschitz on the closed unit
ball. The weak chain rule and symmetry $A^\top=A$ give, almost everywhere,
\[
 \nabla\Delta(v)=2\delta A m(v),\qquad
 \nabla h(v)=2\delta\rho'(\Delta(v))A m(v).
\]
Applying~\eqref{eq:dtfo-weak-gauss} from
Lemma~\ref{lem:gradient_extraction}, we obtain
\begin{align*}
 \mathbb E\widehat g
 &=-\frac{A^{-1}}{2\gamma\delta}\mathbb E_V\nabla h(V)\\
 &=-\frac1\gamma\mathbb E_V[\rho'(\Delta(V))m(V)]\\
 &=\bar g-\frac1\gamma\mathbb E_V[
      \{\rho'(\Delta(V))+\gamma\}m(V)].
\end{align*}
This is the claimed identity with $B=-\gamma^{-1}\mathbb E_V[ \{\rho'(\Delta(V))+\gamma\}m(V)]$.
The inverse matrix cancels before any norm bound, and no derivative of $f$ at a sphere query point is required.

For the spherical specialization $A=I_d$, since $o\in\{-1,1\}$ and $\|U\|_2=1$, the estimator has exact magnitude
\[
 \|\widehat g\|_2
 =\left\|-\frac d{2\gamma\delta}oU\right\|_2
 =\frac d{2\gamma\delta}.
\]
Finally, Assumption~\ref{ass:oracle_general} implies $|\rho'(z)+\gamma|\le M|z|$ on $[-2G,2G]$. Therefore,
\begin{align*}
 \|B\|_2
 &\le\frac1\gamma\mathbb E_V[
       |\rho'(\Delta(V))+\gamma|\,\|m(V)\|_2]\\
 &\le\frac{MG}{\gamma}\mathbb E_V|\Delta(V)|
 \le\frac{2MG^2\delta}{\gamma}\mathbb E\|V\|_2
 =\frac{2d}{d+1}\frac{MG^2}{\gamma}\delta.
\end{align*}
For $A=I_d$, symmetry of $V$ also gives $\bar g=\mathbb E_V\nabla f(x+\delta V)$, as used in the transfer theorem.
\end{proof}

\section{Proof of Theorem~\ref{thm:query_regret_transfer_interval}}
\label{sec:proof_of_thm:query_regret_transfer_interval}

We first project the comparator path into $\mathcal K_\delta$ without increasing
its path length. Curvature preservation and Lemma~\ref{lem:gradient_estimator}
then transfer surrogate regret; the final step adds smoothing and projection
errors to obtain $C_{\mathrm{sph}}\delta|I|$.
\begin{proof}
Fix an interval $I=[a,b]$ and a comparator sequence
$(w_t^*)_{t\in I}$ independently of the learner's randomization.
Set $v_t=\Pi_{\mathcal K_\delta}(w_t^*)$ and write
$\mathbb E_t[\cdot]=\mathbb E[\cdot\mid\mathcal F_{t-1}]$.
The loss sequence is fixed, $w_t$ is measurable with respect to
$\mathcal F_{t-1}$, and the comparators and their projections are fixed.
By Lemma~\ref{lem:gradient_estimator},
\[
 \mathbb E_t\widehat g_t=\bar g_t+B_t,
 \qquad \bar g_t=\mathbb E_V\nabla f_t(w_t+\delta V),
 \qquad \|B_t\|_2\le\frac{2d}{d+1}\frac{MG^2}{\gamma}\delta.
\]

Lemma~\ref{lem:smoothed_convexity} gives the supporting inequality
with curvature. Since $w_t,v_t\in\mathcal K_\delta\subseteq\mathcal K$,
we have $\|w_t-v_t\|_2\le D$, and hence
\begin{align*}
 \widetilde f_t(w_t)-\widetilde f_t(v_t)
 &\le\langle\bar g_t,w_t-v_t\rangle
       -\frac\sigma2\|w_t-v_t\|_2^2\\
 &=\mathbb E_t\!\left[
     \langle\widehat g_t,w_t-v_t\rangle
       -\frac\sigma2\|w_t-v_t\|_2^2\right]
       -\langle B_t,w_t-v_t\rangle\\
 &\le\mathbb E_t[q_t(w_t)-q_t(v_t)]
       +\frac{2d}{d+1}\frac{MG^2D}{\gamma}\delta.
\end{align*}
Taking total expectations and summing gives
\begin{equation}
 \mathbb E\sum_{t\in I}
       [\widetilde f_t(w_t)-\widetilde f_t(v_t)]
 \le\mathbb E\sum_{t\in I}[q_t(w_t)-q_t(v_t)]
       +\frac{2d}{d+1}\frac{MG^2D}{\gamma}\delta|I|.
 \label{eq:base_regret_surrogate}
\end{equation}

To bound the regret on the queried points, define
$Q_t(w)=\mathbb E_U[f_t(w+\delta U)+f_t(w-\delta U)]/2$.
Since $w_t$ is chosen before $U_t$ is drawn,
\[
 \mathbb E_t\left[\frac{f_t(w_t^+)+f_t(w_t^-)}2\right]
 =Q_t(w_t).
\]
Adding and subtracting the smoothed losses and the loss at the
projected comparator gives
\begin{align*}
 Q_t(w_t)-f_t(w_t^*)
 &=\bigl(Q_t(w_t)-\widetilde f_t(w_t)\bigr)
   +\bigl(\widetilde f_t(w_t)-\widetilde f_t(v_t)\bigr)\\
 &\quad+\bigl(\widetilde f_t(v_t)-f_t(v_t)\bigr)
   +\bigl(f_t(v_t)-f_t(w_t^*)\bigr).
\end{align*}
By Lemma~\ref{lem:smoothing_approximation}, the first and third
terms together are at most
$G\delta/(d+1)+dG\delta/(d+1)=G\delta$.

For the fourth term, the domain assumptions imply
$(1-\delta)w_t^*\in\mathcal K_\delta$: for every
$z\in\mathbb B_d\subseteq\mathcal K$, convexity gives
$(1-\delta)w_t^*+\delta z\in\mathcal K$.
Since $v_t$ is the nearest point in $\mathcal K_\delta$,
\[
 \|v_t-w_t^*\|_2
 \le\|(1-\delta)w_t^*-w_t^*\|_2
 =\delta\|w_t^*\|_2\le D\delta.
\]
Here $\|w_t^*\|_2\le D$ follows from $0,w_t^*\in\mathcal K$.
Thus Lipschitzness bounds the fourth term by $GD\delta$.

Taking expectations, summing over $I$, and substituting
\eqref{eq:base_regret_surrogate} proves
\begin{align*}
 &\mathbb E\!\left[\sum_{t\in I}
     \left(\frac{f_t(w_t^+)+f_t(w_t^-)}2-f_t(w_t^*)\right)\right]\\
 &\qquad\le\mathbb E\sum_{t\in I}[q_t(w_t)-q_t(v_t)]
   +\left(G(D+1)+\frac{2d}{d+1}\frac{MG^2D}{\gamma}\right)
       \delta|I|.
\end{align*}
This is the stated bound with $C_{\mathrm{sph}}$.
Finally, Euclidean projection onto the closed convex set $\mathcal K_\delta$ is nonexpansive, so
\[
 \sum_{t=a+1}^b\|v_t-v_{t-1}\|_2
 \le\sum_{t=a+1}^b\|w_t^*-w_{t-1}^*\|_2.
\]
This also proves the comparator-path assertion.

If the base algorithm has the pathwise guarantee
$\sum_{t\in I}[q_t(w_t)-q_t(v_t)]
\le R_{\mathcal A}(v_{a:b},I)$, its expectation can be substituted
into the preceding bound. This requires the guarantee on the realized
surrogate sequence, even though $q_t$ depends on current feedback.
The argument applies for every $\sigma\ge0$.
\end{proof}

\section{Algorithms and Proofs for the DTFO Consequences}
\label{sec:proofs_dtfo_consequences}

Sections~\ref{sec:orgd_appendix} and~\ref{sec:dynamic_orgd_appendix} give ORGD
and Dynamic ORGD, together with their static, interval, and path-length bounds.
Section~\ref{sec:sc-ogd-appendix} proves the SC-OGD curvature cancellation.
Section~\ref{sec:sc_iflh_appendix} specifies IFLH and combines exp-concavity,
expert lifetimes, and interval covering. Each analysis ends by applying
Theorem~\ref{thm:query_regret_transfer_interval}.

\subsection{ORGD: implementation and static/adaptive regret}
\label{sec:orgd_appendix}

We state the projected update and then telescope its distance potential on an
arbitrary interval to prove Theorem~\ref{thm:orgd_static}.

\subsubsection{Implementation}

Algorithm~\ref{alg:orgd} implements DTFO with projected OGD; the proof immediately below covers the convex step-size schedule.

\begin{algorithm}[H]
\caption{Online Relative Gradient Descent (ORGD)}
\label{alg:orgd}
\begin{algorithmic}[1]
\Require Convex compact domain \(\mathcal K\subset\mathbb R^d\),
    horizon \(T\), perturbation radius \(0<\delta\le1/2\).
\Require Known slope magnitude \(\gamma=-\rho'(0)>0\),
    positive step sizes \(\{\eta_t\}_{t=1}^T\).
\State Define \(\mathcal K_\delta=
    \{w\in\mathcal K:w+\delta\mathbb B_d\subseteq\mathcal K\}\).
\State Initialize \(w_1\in\mathcal K_\delta\) arbitrarily.
\For{\(t=1,\ldots,T\)}
    \State Sample \(u_t\sim\operatorname{Unif}(\mathbb S^{d-1})\).
    \State Set \(w_t^+=w_t+\delta u_t\) and \(w_t^-=w_t-\delta u_t\).
    \State Query the dueling oracle with \((w_t^+,w_t^-)\)
        and receive \(o_t\in\{-1,+1\}\).
    \State Construct \(\hat g_t=-\dfrac{d}{2\gamma\delta}o_tu_t\).
    \State Update \(w_{t+1}=
        \Pi_{\mathcal K_\delta}(w_t-\eta_t\hat g_t)\).
\EndFor
\end{algorithmic}
\end{algorithm}

\subsubsection{Proof of Theorem~\ref{thm:orgd_static}}
\label{sec:proof-orgd}

\begin{proof}
Fix \(I=[s,e]\) and \(v\in\mathcal K_\delta\). Write \(a_t=\|w_t-v\|_2^2\). By the non-expansiveness of Euclidean projection, since $v\in K_\delta$,
\[
\begin{aligned}
a_{t+1}
&=\|w_{t+1}-v\|_2^2
=\left\|\Pi_{K_\delta}(w_t-\eta_t\hat g_t)-v\right\|_2^2
\le \left\|w_t-\eta_t\hat g_t-v\right\|_2^2\\
&=\|w_t-v\|_2^2
   -2\eta_t\langle \hat g_t,w_t-v\rangle
   +\eta_t^2\|\hat g_t\|_2^2\\
&=a_t
   -2\eta_t\langle \hat g_t,w_t-v\rangle
   +\eta_t^2\|\hat g_t\|_2^2.
\end{aligned}
\]
Rearranging and dividing by $2\eta_t>0$ gives
\[
\langle\hat g_t,w_t-v\rangle
\le \frac{a_t-a_{t+1}}{2\eta_t}
+\frac{\eta_t}{2}\|\hat g_t\|_2^2.
\]
Since \(\eta_t\) is non-increasing and \(a_t\le D^2\),
\begin{align*}
 \sum_{t=s}^{e}\frac{a_t-a_{t+1}}{2\eta_t}
 &=\frac{a_s}{2\eta_s}
   +\sum_{t=s+1}^{e}a_t
      \left(\frac{1}{2\eta_t}-\frac{1}{2\eta_{t-1}}\right)
   -\frac{a_{e+1}}{2\eta_e} 
   \le\frac{D^2}{2\eta_e}.
\end{align*}
Using \(\|\hat g_t\|_2\le G_{\mathrm{surr}}\), we obtain the pathwise interval bound
\begin{align*}
 \sum_{t=s}^{e}\langle\hat g_t,w_t-v\rangle
 &\le\frac{D^2}{2\eta_e}
       +\frac{G_{\mathrm{surr}}^2}{2}\sum_{t=s}^{e}\eta_t 
       =\frac{DG_{\mathrm{surr}}}{2}
       \left(\sqrt e+\sum_{t=s}^{e}\frac{1}{\sqrt t}\right)
 \le\frac32DG_{\mathrm{surr}}\sqrt T,
\end{align*}
where the equality results from substituting \(\eta_t=\frac{D}{G_{\mathrm{surr}}\sqrt t}\) and, then we used \(\sum_{t=1}^{e}t^{-1/2}\le2\sqrt e\) and \(e\le T\).
The algorithm is not restarted at \(s\) so the starting point \(w_s\) is the iterate of the continuing run.

For the fixed loss sequence, choose \(w_I^*\in\arg\min_{w\in\mathcal K}\sum_{t\in I}f_t(w)\) and set \(v=\Pi_{\mathcal K_\delta}(w_I^*)\). Applying Theorem~\ref{thm:query_regret_transfer_interval} to the preceding surrogate bound and substituting \(G_{\mathrm{surr}}=d/(2\gamma\delta)\) gives the claimed expected query-regret inequality. The tuning \(\delta=cT^{-1/4}\) yields the adaptive rate, and taking \(I=[1,T]\) gives the static rate for the same algorithm and parameters.
\end{proof}

\subsection{Dynamic ORGD: implementation and dynamic regret}
\label{sec:dynamic_orgd_appendix}

The geometric step-size grid and aggregation rule below define Dynamic ORGD.
The proof of Theorem~\ref{thm:dynamic_orgd} combines an expert path-length bound
with the aggregation penalty before applying DTFO.

\subsubsection{Implementation}

We describe the implementation of Dynamic ORGD before proving its regret bound.
Following \cite{zhang2018adaptive}, set \(N=\left\lceil\frac12\log_2\left(1+\frac{4T}{7}\right)\right\rceil+1,\) \(\eta_i=\frac{2^{i-1}D}{G_{\mathrm{surr}}}\sqrt{\frac{7}{2T}},\) for \(i\in [N]\), and use the meta learning rate and initial weights \(\lambda=\sqrt{\frac{2}{TG_{\mathrm{surr}}^2D^2}},\) \(p_{1,i}=\frac{N+1}{N\,i(i+1)},\) which sum to one. 
Centering the shared linear surrogate at \(w_t\) changes neither its regret nor the normalized exponential-weight update.

\begin{algorithm}[H]
\caption{Dynamic ORGD (via Improved Ader)}
\label{alg:dynamic_orgd}
\begin{algorithmic}[1]
\Require Convex compact domain \(\mathcal K\subset\mathbb R^d\)
    containing \(\mathbb B_d\), known diameter bound \(D\), Horizon \(T\), radius \(0<\delta\le1/2\), known slope magnitude \(\gamma=-\rho'(0)>0\).
\State Define \(\mathcal K_\delta=
    \{w\in\mathcal K:w+\delta\mathbb B_d\subseteq\mathcal K\}\).
\State Choose \(N,\{\eta_i\}_{i=1}^N,\lambda,\{p_{1,i}\}_{i=1}^N\) as stated above.
\State Initialize \(w_{1,i}=0\) for \(i=1,\ldots,N\).
\For{\(t=1,\ldots,T\)}
    \State Form the mixture \(w_t=\sum_{i=1}^N p_{t,i}w_{t,i}\).
    \State Sample \(u_t\sim\operatorname{Unif}(\mathbb S^{d-1})\);
        set \(w_t^\pm=w_t\pm\delta u_t\).
    \State Query \((w_t^+,w_t^-)\), receive
        \(o_t\in\{-1,+1\}\), and set
        \(\hat g_t=-\dfrac{d}{2\gamma\delta}o_tu_t\).
    \State Compute \(z_{t,i}=\langle\hat g_t,w_{t,i}-w_t\rangle\)
        for \(i=1,\ldots,N\).
    \For{\(i=1,\ldots,N\)}
        \State Update
            \(p_{t+1,i}=\dfrac{p_{t,i}\exp(-\lambda z_{t,i})}
            {\sum_{j=1}^N p_{t,j}\exp(-\lambda z_{t,j})}\).
        \State Update
            \(w_{t+1,i}=\Pi_{\mathcal K_\delta}
            (w_{t,i}-\eta_i\hat g_t)\).
    \EndFor
\EndFor
\end{algorithmic}
\end{algorithm}

Improved Ader adapts to unknown path length by aggregating experts with different step sizes. All experts share the surrogate generated by one duel at a common radius. This construction does not simultaneously provide the feedback that would arise from querying other exploration radii. Changing the radius changes the queried pair and the distribution of its comparison bit. We therefore use the path-length-independent choice \(\delta=cT^{-1/4}\) in this instantiation. 

\subsubsection{Proof of Theorem~\ref{thm:dynamic_orgd}}
\label{sec:proof-dynamic-orgd}

\begin{proof}
Fix \(W^*\) and set \(v_t=\Pi_{\mathcal K_\delta}(w_t^*)\).
The non-expansive property of Euclidean projection implies that, for every
$t=1,\ldots,T-1$,
\[
\|v_{t+1}-v_t\|_2
=
\left\|
\Pi_{\mathcal K_\delta}(w_{t+1}^*)
-
\Pi_{\mathcal K_\delta}(w_t^*)
\right\|_2
\le
\|w_{t+1}^*-w_t^*\|_2.
\]
Therefore, summing both sides,
\[
P_T(v_{1:T})
=
\sum_{t=1}^{T-1}\|v_{t+1}-v_t\|_2
\le
\sum_{t=1}^{T-1}\|w_{t+1}^*-w_t^*\|_2
=
P_T(W^*)
=
P.
\]
The domain \(\mathcal K_\delta\) is compact and convex, contains zero,
and has diameter at most \(D\). Thus the experts and their mixture
remain feasible, and both queries belong to \(\mathcal K\).
For every realization of the feedback, the surrogate losses are linear
with \(\|\nabla q_t(w)\|_2=G_{\mathrm{surr}}\).
The following expert and meta bounds establish the surrogate guarantee
pathwise, as in the proof of Theorem~4 of \cite{zhang2018adaptive}.

For each expert, non-expansiveness of Euclidean projection gives, for every
$t$,
\[
\begin{aligned}
\|w_{t+1,i}-v_t\|_2^2
&\le
\|w_{t,i}-\eta_i\hat g_t-v_t\|_2^2
=
\|w_{t,i}-v_t\|_2^2
-2\eta_i\langle \hat g_t,w_{t,i}-v_t\rangle
+\eta_i^2\|\hat g_t\|_2^2.
\end{aligned}
\]
Hence,
\[
\langle \hat g_t,w_{t,i}-v_t\rangle
\le
\frac{\|w_{t,i}-v_t\|_2^2
-\|w_{t+1,i}-v_t\|_2^2}{2\eta_i}
+\frac{\eta_i}{2}\|\hat g_t\|_2^2.
\]
Summing over $t=1,\ldots,T$, the potential terms satisfy
\[
\begin{aligned}
&\sum_{t=1}^T
\left(
\|w_{t,i}-v_t\|_2^2
-\|w_{t+1,i}-v_t\|_2^2
\right) \\
&\qquad=
\|w_{1,i}-v_1\|_2^2
-\|w_{T+1,i}-v_T\|_2^2
+\sum_{t=2}^T
\left(
\|w_{t,i}-v_t\|_2^2
-\|w_{t,i}-v_{t-1}\|_2^2
\right).
\end{aligned}
\]
For each $t\ge2$,
\[
\begin{aligned}
&\|w_{t,i}-v_t\|_2^2
-\|w_{t,i}-v_{t-1}\|_2^2
=
\left\langle
(w_{t,i}-v_t)+(w_{t,i}-v_{t-1}),
v_{t-1}-v_t
\right\rangle
\le
2D\|v_t-v_{t-1}\|_2,
\end{aligned}
\]
since all three points lie in $\mathcal K_\delta$, whose diameter is at
most $D$. Moreover,
$\|w_{1,i}-v_1\|_2^2\le D^2$, while the terminal potential is
non-positive. Therefore, using
$P_T(v_{1:T})\le P$ and
$\|\hat g_t\|_2\le G_{\mathrm{surr}}$,
\begin{equation}
\label{eq:dynamic_orgd_expert_bound}
 \sum_{t=1}^T\langle\hat g_t,w_{t,i}-v_t\rangle
 \le\frac{D^2+2DP}{2\eta_i}
       +\frac{\eta_iG_{\mathrm{surr}}^2T}{2}
 \le\frac{7D^2+4DP}{4\eta_i}
       +\frac{\eta_iG_{\mathrm{surr}}^2T}{2}.
\end{equation}

For the meta algorithm, recall that \(z_{t,i}=\langle \hat g_t,w_{t,i}-w_t\rangle\) and \(|z_{t,i}| \le \|\hat g_t\|_2\|w_{t,i}-w_t\|_2 \le DG_{\mathrm{surr}}\) because both $w_{t,i}$ and $w_t$ lie in $\mathcal K_\delta$ with diameter $D$ and $\|\hat g_t\|_2\le G_{\mathrm{surr}}$.
Moreover, because $w_t=\sum_i p_{t,i}w_{t,i}$ and $\sum_i p_{t,i}=1$, \(\sum_i p_{t,i}z_{t,i}
= \left\langle \hat g_t, \sum_i p_{t,i}(w_{t,i}-w_t) \right\rangle =0.\)
Thus, regarding $z_{t,i}$ as a random variable taking value $z_{t,i}$ with probability $p_{t,i}$, it has mean zero and lies in $[-DG_{\mathrm{surr}},DG_{\mathrm{surr}}]$. Hence, by Hoeffding's lemma,
\[
\log\left(\sum_i p_{t,i}e^{-\lambda z_{t,i}} \right)
\le \frac{\lambda^2D^2G_{\mathrm{surr}}^2}{2}.
\]

Define the normalization factor \(Z_t:=\sum_j p_{t,j}\exp(-\lambda z_{t,j}).\)
Given the exponential-weights update \(p_{t+1,i} = \frac{p_{t,i}\exp(-\lambda z_{t,i})}{Z_t}\), taking the \(\log\) on both sides, rearranging and summing over $t$ yields
\[
-\lambda\sum_{t=1}^T z_{t,i}
=
\log p_{T+1,i}
-\log p_{1,i}
+\sum_{t=1}^T\log Z_t.
\]
Since $p_{T+1,i}\le 1$, we have $\log p_{T+1,i}\le 0$, and therefore
\[
-\lambda\sum_{t=1}^T z_{t,i}
\le
\ln\frac{1}{p_{1,i}}
+\sum_{t=1}^T\log Z_t
\le
\ln\frac{1}{p_{1,i}}
+
\frac{\lambda^2D^2G_{\mathrm{surr}}^2T}{2}.
\]
where the second inequality results from replacing \(Z_t\) and applying Hoeffding's lemma.
Dividing by $\lambda$ and using \(-z_{t,i} = \langle \hat g_t,w_t-w_{t,i}\rangle,
\)
we obtain
\begin{align}
\label{eq:dynamic_orgd_meta_bound}
 \sum_{t=1}^T\langle\hat g_t,w_t-w_{t,i}\rangle
 \le\frac{\ln(1/p_{1,i})}{\lambda}
    +\frac{\lambda D^2G_{\mathrm{surr}}^2T}{2}
    \le\frac{DG_{\mathrm{surr}}\sqrt{2T}}{2}
       \bigl[1+2\ln(i+1)\bigr],
\end{align}
where \(\ln(1/p_{1,i})\le2\ln(i+1)\) and substituting \(\lambda = \sqrt{\frac{2}{T G_{\mathrm{surr}}^2D^2}}\).
Now define the index $ k= \left\lfloor \frac12\log_2\left(1+\frac{4P}{7D}\right) \right\rfloor+1, $ and the number of experts $ N= \left\lceil \frac12\log_2\left(1+\frac{4T}{7}\right) \right\rceil+1. $
Since \(P\le D(T-1)\), the index \(k\) lies in \(\{1,\ldots,N\}\), and the grid satisfies
\[
 \eta_k\le\eta^*<2\eta_k,
 \qquad
 \eta^*=\sqrt{\frac{7D^2+4DP}{2TG_{\mathrm{surr}}^2}},
\]
because $ k-1 \le \frac12\log_2\left(1+\frac{4P}{7D}\right) < k $ from the definition of \(k\).
Substituting \(i=k\) in \eqref{eq:dynamic_orgd_expert_bound} gives
an expert regret of at most
\(\frac{3G_{\mathrm{surr}}}{4}\sqrt{2T(7D^2+4DP)}\).
Adding \eqref{eq:dynamic_orgd_meta_bound} bounds
\(\sum_{t=1}^T\langle\hat g_t,w_t-v_t\rangle\) by the first two
terms in the bound below.

These potential inequalities hold for each realized surrogate sequence,
even though it depends on the learner's predictions and observations.
In particular, the base argument does not require \(\hat g_t\) to be
an unbiased gradient of the original loss. Taking expectations and
applying Theorem~\ref{thm:query_regret_transfer_interval} adds
\(C_{\mathrm{sph}}\delta T\), which accounts for estimator bias,
smoothing, and comparator projection. Thus the exact query-regret
bound is
\begin{align}
 \operatorname{Regret}_D(W^*,T)
 &\le \frac{3G_{\mathrm{surr}}}{4}
       \sqrt{2T(7D^2+4DP)}\notag\\
 &\quad+\frac{DG_{\mathrm{surr}}\sqrt{2T}}{2}
       \bigl[1+2\ln(k+1)\bigr]+C_{\mathrm{sph}}\delta T.
 \label{eq:dynamic_orgd_exact_bound}
\end{align}
To obtain the order stated in Theorem~\ref{thm:dynamic_orgd}, set
\(x=P/D\ge0\). Using \(\ln(1+y)\le y\),
\[
 k+1\le 2+\frac{\ln(1+4x/7)}{2\ln2}
 \le2+\frac{2x}{7\ln2}\le2+x.
\]
Since \(\ln y\le\sqrt y\) for \(y\ge1\),
\(1+2\ln(k+1)\le(1+2\sqrt2)\sqrt{1+x}\).
Also \(7D^2+4DP\le7D^2(1+x)\), so both first-order terms in
\eqref{eq:dynamic_orgd_exact_bound} are
\(\mathcal O(DG_{\mathrm{surr}}\sqrt{T(1+P/D)})\).
Substituting \(G_{\mathrm{surr}}=d/(2\gamma\delta)\) and the single
radius \(\delta=cT^{-1/4}\), \(0<c\le1/2\), completes the proof.
\end{proof}

\subsection{SC-OGD: proof of Theorem~\ref{thm:sc_spherical_static}}
\label{sec:sc-ogd-appendix}

Algorithm~\ref{alg:orgd} with $\eta_t=1/(\sigma t)$ implements SC-OGD.
The quadratic surrogate cancels the increasing inverse step size, leaving
a harmonic regret bound to which Theorem~\ref{thm:query_regret_transfer_interval} applies.

\begin{proof}
We reproduce the strongly convex OGD argument underlying
\cite[Theorem~1]{hazan2007logarithmic}. Fix \(v\in\mathcal K_\delta\),
write \(a_t=\|w_t-v\|_2^2\), and set \(\eta_t=1/(\sigma t)\).
The projection inequality gives
\begin{align*}
 q_t(w_t)-q_t(v)
 &=\langle\hat g_t,w_t-v\rangle-\frac{\sigma}{2}a_t\\
 &\le\frac{a_t-a_{t+1}}{2\eta_t}
       -\frac{\sigma}{2}a_t+\frac{\eta_t}{2}\|\hat g_t\|_2^2\\
 &=\frac{\sigma}{2}\bigl((t-1)a_t-ta_{t+1}\bigr)
       +\frac{\|\hat g_t\|_2^2}{2\sigma t}.
\end{align*}
The distance terms telescope to \(-\sigma T a_{T+1}/2\le0\).
Using \(\|\hat g_t\|_2\le d/(2\gamma\delta)\) yields the pathwise bound
\[
 \sum_{t=1}^T\bigl(q_t(w_t)-q_t(v)\bigr)
 \le\frac{G_{\mathrm{surr}}^2}{2\sigma}\mathsf H_T
 \le\frac{G_{\mathrm{surr}}^2}{2\sigma}(1+\log T),
 \qquad \mathsf H_T=\sum_{t=1}^T\frac1t.
\]
The surrogates also satisfy the global gradient bound
\(\|\nabla q_t(w)\|_2\le d/(2\gamma\delta)+\sigma D\).
The sharper bound above uses only gradients at the played points and
remains valid when the surrogates depend on the learner's observations.

Choose a fixed \(w^*\in\arg\min_{w\in\mathcal K}\sum_{t=1}^T f_t(w)\).
Apply Theorem~\ref{thm:query_regret_transfer_interval} with
\(I=[1,T]\) and \(v=\Pi_{\mathcal K_\delta}(w^*)\) to obtain
\[
 \operatorname{Regret}_S(T)
 \le\frac{G_{\mathrm{surr}}^2}{2\sigma}\mathsf H_T
       +C_{\mathrm{sph}}\delta T.
\]
With \(\ell_T=\log(eT)\), choose
\(\delta=c(\ell_T/T)^{1/3}\), \(0<c\le1/2\).
Since \(\ell_T\le T\) and \(\mathsf H_T\le\ell_T\), this is a
feasible radius and both terms are
\(\mathcal O(T^{2/3}\ell_T^{1/3})\) for fixed problem parameters.
\end{proof}

\subsection{IFLH with SC-OGD experts: implementation and adaptive regret}
\label{sec:sc_iflh_appendix}

This subsection specifies the implementation underlying Theorem~\ref{thm:sc_spherical_weakly_adaptive}. It uses the binary specialization of IFLH \cite[Algorithm~1 and Lemma~1]{zhang2018dynamic} on the shared quadratic surrogates \(q_t\), not on the original losses.
We first establish exp-concavity and specify expert lifetimes, then give the implementation and its pathwise interval bound. The final step transfers this bound to query regret and tunes the radius.

\subsubsection{Surrogate properties and expert lifetimes}

Recall \(H_\delta=d/(2\gamma\delta)+\sigma D\), and define
\(\alpha_\delta=\sigma/H_\delta^2\).
For every \(w\in\mathcal K_\delta\),
\[
 \|\nabla q_t(w)\|_2
 \le\|\hat g_t\|_2+\sigma\|w-w_t\|_2
 \le H_\delta.
\]
Moreover,
\[
 \alpha_\delta\nabla q_t(w)\nabla q_t(w)^\top
 \preceq\alpha_\delta H_\delta^2 I_d
 =\sigma I_d=\nabla^2q_t(w).
\]
It follows by differentiating \(\exp(-\alpha_\delta q_t(w))\) that this function is concave. Hence \(q_t\) is \(\alpha_\delta\)-exp-concave, as in \cite[Lemma~2]{zhang2018dynamic}. This property concerns the explicitly known quadratic and it does not assume smoothness of the original \(f_t\).

For an integer \(j\ge1\), define
\begin{equation}
 \nu_2(j)=\max\{r\ge0:2^r\text{ divides }j\},
 \qquad
 e_j=j+2^{\nu_2(j)}.
 \label{eq:sc_iflh_lifetime}
\end{equation}
Expert \(j\) starts at round \(j\) and is active on \([j,e_j-1]\), truncated at \(T\). This is the base-\(2\) ending-time rule of IFLH. At a given time, there is at most one active expert at
each lifetime scale, so
\[
 |\mathcal S_t|\le\lfloor\log_2 t\rfloor+1.
\]
Each expert starts at zero, which belongs to \(\mathcal K_\delta\) under the unit-ball assumption and \(\delta<1\).

\subsubsection{Complete implementation}

The complete implementation is presented in Algorithm~\ref{alg:sc_iflh_dtfo}.
The meta algorithm uses learning rate \(\alpha_\delta\) and combines experts started at different times. An expert starting at \(j\) uses local step size \(1/[\sigma(t-j+1)]\) and the update \(w_{t+1,j}=\Pi_{\mathcal K_\delta}\left(w_{t,j}-\frac{\hat g_t+\sigma(w_{t,j}-w_t)}{\sigma(t-j+1)}\right).\)
Write \(p_{t,j}\) for the weights used to form the round-\(t\) mixture. The posterior weights \(\bar p_{t+1,j}\) are normalized after observing round \(t\)'s surrogate, before any experts are removed at round \(t+1\). If no older expert survives, the new expert receives weight one.  Otherwise it receives weight \(1/t\), with the remaining mass assigned to survivors. This convention includes \(t=1\) and prevents normalization of an empty set.

\begin{algorithm}[H]
\caption{Spherical DTFO with IFLH and SC-OGD Experts}
\label{alg:sc_iflh_dtfo}
\begin{algorithmic}[1]
\Require Convex compact \(\mathcal K\subset\mathbb R^d\) containing
    \(\mathbb B_d\), diameter bound \(D\), horizon \(T\).
\Require Radius \(0<\delta\le1/2\), known
    \(\sigma>0\) and \(\gamma=-\rho'(0)>0\).
\State Set \(\mathcal K_\delta=
    \{w\in\mathcal K:w+\delta\mathbb B_d\subseteq\mathcal K\}\).
\State Set \(H_\delta=d/(2\gamma\delta)+\sigma D\),
    \(\alpha_\delta=\sigma/H_\delta^2\), and \(\mathcal S_0=\varnothing\).
\For{\(t=1,\ldots,T\)}
    \State Keep survivors
        \(\mathcal A_t=\{j\in\mathcal S_{t-1}:t<e_j\}\);
        discard other experts.
    \State Start expert \(t\) with \(w_{t,t}=0\) and
        \(e_t=t+2^{\nu_2(t)}\).
    \If{\(\mathcal A_t=\varnothing\)}
        \State Set \(p_{t,t}=1\).
    \Else
        \State Set \(Z_t^{\mathrm{surv}}
            =\sum_{j\in\mathcal A_t}\bar p_{t,j}\).
        \State Set \(p_{t,j}=(1-1/t)\bar p_{t,j}/Z_t^{\mathrm{surv}}\)
            for \(j\in\mathcal A_t\), and \(p_{t,t}=1/t\).
    \EndIf
    \State Set \(\mathcal S_t=\mathcal A_t\cup\{t\}\) and
        \(w_t=\sum_{j\in\mathcal S_t}p_{t,j}w_{t,j}\).
    \State Sample \(u_t\sim\operatorname{Unif}(\mathbb S^{d-1})\)
        and set \(w_t^\pm=w_t\pm\delta u_t\).
    \State Query \((w_t^+,w_t^-)\), receive \(o_t\in\{-1,+1\}\),
        and set \(\hat g_t=-d\,o_tu_t/(2\gamma\delta)\).
    \State For every \(j\in\mathcal S_t\), compute
        \(z_{t,j}=\langle\hat g_t,w_{t,j}-w_t\rangle
                   +\frac{\sigma}{2}\|w_{t,j}-w_t\|_2^2\).
    \State Set \(Z_t=\sum_{j\in\mathcal S_t}
        p_{t,j}\exp(-\alpha_\delta z_{t,j})\).
    \For{\(j\in\mathcal S_t\)}
        \State Set \(\bar p_{t+1,j}
            =p_{t,j}\exp(-\alpha_\delta z_{t,j})/Z_t\).
        \State Set \(g_{t,j}=\hat g_t+\sigma(w_{t,j}-w_t)\).
        \State Set \(w_{t+1,j}=
            \Pi_{\mathcal K_\delta}
            (w_{t,j}-g_{t,j}/[\sigma(t-j+1)])\).
    \EndFor
\EndFor
\end{algorithmic}
\end{algorithm}

The scores satisfy \(z_{t,j}=q_t(w_{t,j})-q_t(w_t)\), so the centered
update is identical to exponential weighting by \(q_t(w_{t,j})\).
All scores and expert gradients are computed from one duel. Mixtures
remain in \(\mathcal K_\delta\), ensuring feasibility of both queries.
The number of expert projections is at most
\(\lfloor\log_2 t\rfloor+1\) at round \(t\); no extra first-order or
dueling-oracle evaluations are needed.

\subsubsection{Pathwise interval analysis}

Lemma~\ref{lem:sc_iflh_pathwise} combines expert regret, exponential-weight aggregation, and an interval cover; its proof below holds for every realized surrogate sequence.

\begin{lemma}[Interval Regret on the Quadratic Surrogates]
\label{lem:sc_iflh_pathwise}
Let \(m_I=\lceil\log_2|I|\rceil+1\). For every realized feedback sequence generated by Algorithm~\ref{alg:sc_iflh_dtfo}, every interval \(I=[a,b]\), and every \(v\in\mathcal K_\delta\),
\begin{equation}
 \sum_{t=a}^{b}\bigl(q_t(w_t)-q_t(v)\bigr)
 \le\frac{H_\delta^2}{2\sigma}
       m_I(1+3\log T).
 \label{eq:sc_iflh_pathwise}
\end{equation}
\end{lemma}

\begin{proof}
We give the expert, aggregation, and interval-covering arguments explicitly. They specialize the proof of \cite[Theorem~2]{zhang2018dynamic} to the quadratics and the normalized weights above.

Fix an expert \(j\), a time \(b'<e_j\), and \(v\in\mathcal K_\delta\).
Let \(\tau=t-j+1\), \(a_t=\|w_{t,j}-v\|_2^2\), and
\(g_{t,j}=\nabla q_t(w_{t,j})\). Strong convexity and the projection
inequality give
\[
 q_t(w_{t,j})-q_t(v)
 \le\frac{\sigma}{2}
       \bigl((\tau-1)a_t-\tau a_{t+1}\bigr)
       +\frac{\|g_{t,j}\|_2^2}{2\sigma\tau}.
\]
Summing from \(j\) to \(b'\), the initial distance coefficient is zero
and the final distance coefficient is nonpositive. Thus
\begin{equation}
 \sum_{t=j}^{b'}\bigl(q_t(w_{t,j})-q_t(v)\bigr)
 \le\frac{H_\delta^2}{2\sigma}
       \bigl(1+\log(b'-j+1)\bigr).
 \label{eq:sc_iflh_expert_regret}
\end{equation}
The bound uses \(H_\delta\), not merely the smaller norm of the
estimate at the mixture center.

By exp-concavity and \(w_t=\sum_jp_{t,j}w_{t,j}\),
\[
 Z_t=\sum_{j\in\mathcal S_t}
       p_{t,j}\exp\bigl(-\alpha_\delta
                      (q_t(w_{t,j})-q_t(w_t))\bigr)\le1.
\]
Consequently, for each active \(j\),
\[
 \alpha_\delta\bigl(q_t(w_t)-q_t(w_{t,j})\bigr)
 \le\log\bar p_{t+1,j}-\log p_{t,j}.
\]
For \(j<t<e_j\), survivor normalization gives
\(p_{t,j}\ge(1-1/t)\bar p_{t,j}\), since
\(Z_t^{\mathrm{surv}}\le1\). Also \(p_{j,j}\ge1/j\).
For every \(b'\in[j,e_j-1]\cap[1,T]\), summation therefore yields
\begin{align}
 \alpha_\delta\sum_{t=j}^{b'}
       \bigl(q_t(w_t)-q_t(w_{t,j})\bigr)
 &\le\log\bar p_{b'+1,j}-\log p_{j,j}
       +\sum_{t=j+1}^{b'}\log\frac{t}{t-1}\notag\\
 &\le\log j+\log\frac{b'}j=\log b'\le\log T.
 \label{eq:sc_iflh_meta_regret}
\end{align}
The posterior \(\bar p_{b'+1,j}\) is taken before the next pruning step, so this argument also covers the last round of an expert's lifetime.

Starting at \(j_1=a\), form the segments
\[
 [j_r,\min\{e_{j_r}-1,b\}],
 \qquad j_{r+1}=e_{j_r},
\] 
until \(b\) is covered. These form a disjoint partition of \(I\), each starting at the birth of an expert that remains active throughout that segment. Successive untruncated lengths \(2^{\nu_2(j_r)}\) at least double. If there are \(m\) segments, the first \(m-1\) are complete, so \(|I|\ge(2^{m-1}-1)+1=2^{m-1}\). In particular \(m\le m_I\), as in the binary instance of \cite[Lemma~1]{zhang2018dynamic}.

Add \eqref{eq:sc_iflh_expert_regret} and
\eqref{eq:sc_iflh_meta_regret} on each segment. Since
\(\alpha_\delta^{-1}=H_\delta^2/\sigma\), this gives
\[
 \sum_{t=a}^{b}\bigl(q_t(w_t)-q_t(v)\bigr)
 \le\frac{H_\delta^2}{2\sigma}m(1+3\log T)
 \le\frac{H_\delta^2}{2\sigma}
       m_I(1+3\log T).
\]
Every step is a deterministic inequality for
the realized quadratic losses; no unbiasedness or independence of the
surrogate sequence is assumed.
\end{proof}

\subsubsection{Proof of Theorem~\ref{thm:sc_spherical_weakly_adaptive} and radius tuning}
\begin{proof}
The surrogates are \(\sigma\)-strongly convex and their gradients on
\(\mathcal K_\delta\) are bounded by \(H_\delta\).
Lemma~\ref{lem:sc_iflh_pathwise} specializes the IFLH
analysis of \cite[Theorem~2]{zhang2018dynamic} to these surrogates:
for every realization and every \(v\in\mathcal K_\delta\), 
\[
 \sum_{t\in I}\bigl(q_t(w_t)-q_t(v)\bigr)
 \le\frac{H_\delta^2}{2\sigma}
       m_I(1+3\log T).
\]
Choose a fixed minimizer
\(w_I^*\in\arg\min_{w\in\mathcal K}\sum_{t\in I}f_t(w)\), set
\(v=\Pi_{\mathcal K_\delta}(w_I^*)\), and apply
Theorem~\ref{thm:query_regret_transfer_interval} to obtain the stated
query-regret bound with penalty \(C_{\mathrm{sph}}\delta|I|\).

For the uniform query-regret rate, let
\(\ell_T=\log(eT)\) and
\(\kappa=d/(2\gamma)+\sigma D\), and choose the single radius
\[
 \delta=c\min\{1,(\ell_T^2/T)^{1/3}\},\qquad 0<c\le1/2.
\]
Since \(H_\delta\le\kappa/\delta\) and
\(m_I(1+3\log T)=\mathcal O(\ell_T^2)\) uniformly over intervals,
Theorem~\ref{thm:query_regret_transfer_interval} gives
\[
 \max_I\operatorname{Regret}_A(I)
 \le\mathcal O(\ell_T^2/\delta^2)+C_{\mathrm{sph}}\delta T.
\]
When \(\ell_T^2\le T\), substituting the chosen radius yields
\(\mathcal O(T^{2/3}\ell_T^{2/3})\).
When \(\ell_T^2>T\), the elementary bound
\(\operatorname{Regret}_A(I)\le GD|I|\le GDT
 \le GD T^{2/3}\ell_T^{2/3}\)
gives the same order. Thus one radius, fixed before the run and shared
by all experts, gives the stated rate for every horizon.
\end{proof}

\section{Proofs for smooth losses and ellipsoidal FTRL}
\label{sec:proofs_ellipsoidal}

Section~\ref{sec:ellipsoid-setup} introduces the barrier geometry and conditional notation.
Sections~\ref{sec:ellipsoid-feasibility-proof}--\ref{sec:proof-ellipsoid-feasible-comparison}
establish feasibility, quadratic estimator bias, and the one-round transfer.
Section~\ref{sec:ellipsoid-ftrl-proof} proves the pathwise FTRL bound;
Section~\ref{sec:proof-ellipsoidal-master} combines these ingredients and handles
boundary comparators. Sections~\ref{sec:proof-ellipsoidal-convex}
and~\ref{sec:proof-ellipsoidal-strong} verify the two tunings.

\subsection{Setup and barrier geometry}
\label{sec:ellipsoid-setup}
This appendix proves the static query-regret guarantees of Section~\ref{sec:ellipsoidal}. Throughout, the losses are fixed before the learner's randomization, $G$-Lipschitz, $L$-smooth, and $\sigma$-strongly convex on $\K$, with $\sigma\ge0$.
The link satisfies Assumption~\ref{ass:oracle_general}, is odd and three times continuously differentiable on $[-2G,2G]$, and obeys $|\rho'''|\le M_3$. Algorithm~\ref{alg:dueling_ellipsoidal} uses $0<\delta\le\delta_0:=\min\{1/2,1/D\}$ and $0<\eta\le\gamma\delta/d$; $\eta$ is fixed throughout each run.
For any fixed feasible comparator $u \in \K$, with the query points for round $t$ being $w_t^+,w_t^-\in \K$, write
\[
 R_T(u)=\E\sum_{t=1}^T
 \left[\frac{f_t(w_t^+)+f_t(w_t^-)}2-f_t(u)\right],
\]
and regret $R_T$ as 
\[
R_T=\max_{u\in\K}R_T(u)=\operatorname{Regret}_S(T).
\]

\paragraph{Self-concordant barrier.} The use of self-concordant barrier functions as the FTRL regularizer for zeroth-order online learning was pioneered by \cite{abernethy2008competing} and extended in \cite{saha2011improved} to smooth objectives, achieving $\widetilde{\mathcal O}(T^{2/3})$ regret. Further leveraging the inverse Hessian geometry, \cite{hazan2014bandit} demonstrated how to achieve $\tilde{\mathcal{O}}(\sqrt{T})$ regret bounds for strongly convex and smooth loss functions for bandit convex optimization.
Definition~\ref{def:barrier} records the derivative inequalities used below for query feasibility, FTRL stability, and comparator approximation.
\begin{definition}[Self-concordant barrier]\label{def:barrier}
A convex function $\Phi: \text{int}(\mathcal{K}) \to \mathbb{R}$ is a $\nu$-self-concordant barrier for a bounded convex domain $\mathcal{K} \subset \mathbb{R}^d$ with non-empty interior if it is three times continuously differentiable, blows up at the boundary $\partial\mathcal{K}$ (i.e., $\Phi(x) \to \infty$ as $x \to \partial\mathcal{K}$), and satisfies $\forall x \in \Int \mathcal{K}$ and $\forall h \in \mathbb{R}^d$:
\begin{enumerate}
    \item \textbf{Self-Concordance:} $|D^3\Phi(x)[h,h,h]| \le 2 \left(h^\top \nabla^2 \Phi(x) h  \right)^{3/2}$,
    \item \textbf{Barrier Property:} $|D\Phi(x)[h]| \le \sqrt{\nu} \left(h^\top \nabla^2 \Phi(x) h  \right)^{1/2}$,
\end{enumerate}
where $D$ is the differentiation operator and $h$ specifies the direction, thus $D\Phi(x)[h]=\frac{d}{ds} \Phi (x+sh)|_{s=0} = \langle \nabla \Phi(x), h \rangle$, and $D^3\Phi(x)[h, h, h]=\frac{d^3}{ds^3} \Phi (x+sh)|_{s=0}$.
\end{definition}
As in Section~\ref{sec:ellipsoidal}, we assume $\nabla^2\Phi(x)\succ0$ throughout $\operatorname{int}\K$. Since each realized surrogate is continuous on the compact domain and $\Phi$ diverges at its boundary, every FTRL objective attains its unique minimum in the interior.
In zeroth-order online learning, FTRL with a Euclidean regularizer may place iterates on the boundary of the domain $\mathcal{K}$, where symmetric exploration with a fixed Euclidean radius becomes infeasible. A self-concordant barrier over \(\Int\K\) resolves this by keeping the FTRL iterates in the interior. 
Specifically, for \(x\in \Int \K\), write the local norm defined by the Hessian $\nabla^2 \Phi(x)$ as \(\norm{h}_{\nabla^2\Phi(x)}:= \sqrt{h^\top \nabla^2\Phi(x) h},\) then the corresponding local Dikin ellipsoid 
\[
\mathcal{E}_x(r) = \{y \in \mathbb{R}^d : \norm{y-x}_{\nabla^2\Phi(x)}  \le r\}
\]
is contained in \(\operatorname{int}(\mathcal K)\) whenever \(0<r<1\). This Dikin ellipsoid inclusion property allows the exploration geometry to adapt to the domain while maintaining query feasibility. It is also stated in \cite[Section 2.3, Eq. (3)]{hazan2014bandit}.

\paragraph{Regularized Hessian, Sampling Matrix, and Spectral Norm.} To unify the analysis of smooth convex and smooth strongly-convex objectives, we use the regularized Hessian $H_t=\nabla^2 \Phi(w_t) + \eta \sigma t I$ and construct the sampling matrix $A_t=H_t^{-1/2}$ as its symmetric positive definite inverse square root, where \(\sigma=0\) covers the convex case. 
We write $\op{A_t}:=\max_{\|v\|_2=1}\|A_t v\|_2$ for its Euclidean operator norm, also called the spectral norm. This quantity measures the maximum stretching of a unit vector by \(A_t\), so \(\delta\op{A_t}\) bounds the Euclidean displacement of either query resulting from \(w_t\).

To prove Theorem~\ref{thm:universal_ellipsoidal} for Algorithm~\ref{alg:dueling_ellipsoidal}, we first establish the feasibility of the perturbed query points constructed by the sampling matrix $A_t$, and the conditional expectation of the ellipsoidal gradient estimator relative to the ellipsoidal smoothed gradient. 
For a single-round calculation, condition on the pre-query history
$\mathcal F_{t-1}$ defined in Section~\ref{sec:dtfo}, and write
$x=w_t$, $A=A_t$, $H=H_t$, and $f=f_t$.
These quantities are $\mathcal F_{t-1}$-measurable; we abbreviate
$\mathbb E_t$ by $\E$ in the single-round calculations below.
Let $U$ be uniform on the unit sphere independently of this history,
and let $V$ be an independent auxiliary point uniform on the solid
unit ball. Expectations involving $\widehat g$ include both $U$
and the oracle bit.
Denote the sampled loss difference $\Delta(v)=f(x+\delta Av)-f(x-\delta Av),$ and denote the ellipsoidal smoothed gradient
\[
\bar g:=\E_Vm(V), \quad \text{ where }
 m(v)=\frac{\nabla f(x+\delta Av)+\nabla f(x-\delta Av)}2.
\]
By symmetry, $\bar g=\E_V[\nabla f(x+\delta AV)]=\E_V[\nabla f(x-\delta AV)]$, and it is the gradient at $x$ of the locally defined ball smoothing with $A$ held fixed.

\subsection{Query feasibility and spectral bounds}
\label{sec:ellipsoid-feasibility-proof}

The following lemma establishes query feasibility and the spectral bounds needed for the regret analysis.

\begin{lemma}[Feasibility and Technical Bounds]\label{lem:smooth-feas}
All queries are feasible and $\op{A_t}\le D$. 
If $\sigma>0$, $\op{A_t}^2\le1/(\eta\sigma t)$. 
Every sampled loss difference lies in $[-2G,2G]$.
\end{lemma}
\begin{proof}
Because $H_t\succeq\nabla^2\Phi(w_t) \succ 0$, we have \( A_t^\top \nabla^2\Phi(w_t) A_t \preceq A_t^\top H_t A_t = I \). For \(w_t \in \Int \K\), 
\begin{align*}
    \norm{(w_t+\delta A_t U_t) - w_t}_{\nabla^2\Phi(w_t)}^2
 &=(\delta A_t U_t)^\top \nabla^2\Phi(w_t) (\delta A_t U_t) \\
 &= \delta^2 U_t^\top A_t \nabla^2\Phi(w_t) A_t U_t
 \le \delta^2\norm{U_t}^2<1.
\end{align*}
Thus, by the Dikin ellipsoid inclusion property stated above, all perturbed query points \(w_t^+, w_t^-\in \Int \K\), as the same inequality holds for \(w_t^-\).
The same calculation holds for every $v\in\mathbb B_d$, so
$w_t+\delta A_t\mathbb B_d\subseteq\operatorname{int}\K$.

For $x\in\operatorname{int}\K$, set $B_x=(\nabla^2\Phi(x))^{-1/2}$. For a unit vector $v$ and $0<r<1$,  $\norm{rB_xv}_{\nabla^2\Phi(x)}=r<1$, the Dikin ellipsoid inclusion property therefore gives $x+rB_xv\in\K$. Since \(x \in \K\), the diameter assumption implies
\[r \norm{B_xv}_2 = \norm{(x+rB_xv)-x}_2 \leq D\]
Letting $r$ approach $1$ from below and maximizing over unit vectors gives \(\op{B_x} \leq D\). Because $H_t\succeq\nabla^2\Phi(w_t)$, we have $A_t^2=H_t^{-1} \preceq (\nabla^2\Phi(w_t))^{-1}=B^2_{w_t}$. Thus, \(\op{A_t}\leq D\).
When $\sigma>0$, we have $H_t\succeq\eta\sigma tI$, which yields $H_t^{-1}\preceq \frac 1{\eta\sigma t}I$, and thus \(\op{A_t}^2\leq \frac 1{\eta\sigma t}\).

Finally, since $f_t$ are $G$-Lipschitz continuous, we have $|\Delta_t(v)|\le2\delta G\norm{A_tv}\le2G$ by $\delta D\le1$.
\end{proof}

\subsection{Estimator identity, quadratic bias, and local norm}
\label{sec:ellipsoid-bias-proof}

Lemma~\ref{lem:smooth-bias} uses the comparison identity from Lemma~\ref{lem:gradient_estimator} and the third-derivative link bound to obtain quadratic bias and an exact local-norm identity. The proof below reuses Appendix~\ref{sec:proof_of_lem:gradient_estimator}.

\begin{lemma}[Ellipsoidal Gradient Estimator and Bias Control] \label{lem:smooth-bias}
The conditional expectation of the ellipsoid gradient estimator satisfies
\[
\E[\widehat g]= -\E_V\left[\frac{\rho'(\Delta(V))}{\gamma}m(V)\right] =\bar g+B
\]
with the bias term \(-\frac 1\gamma \E_V\left[\left(\rho'(\Delta(V))+\gamma\right)m(V)\right]\) bounded by 
\[
\norm B_2 \le \frac{2M_3G^3\delta^2}{\gamma(d+2)}\tr(A^2)
\le\frac{2M_3G^3}{\gamma}\delta^2\op A^2.
\]
Moreover, $\norm{\widehat g}_{H^{-1}}^2=d^2/(4\gamma^2\delta^2)$.
\end{lemma}
\begin{proof}
Lemma~\ref{lem:gradient_estimator} applies by
Lemma~\ref{lem:smooth-feas}. Its proof in
Appendix~\ref{sec:proof_of_lem:gradient_estimator} uses the sphere-to-ball identity
\begin{equation}\label{eq:ballibp}
 d\E_U[h(U)U]=\E_V[\nabla_v h(V)],
\end{equation}
as established in~\eqref{eq:dtfo-weak-gauss}, with
$h(v)=\rho(\Delta(v))$.
Applying that comparison identity with the present $x$, $A$, and $f$ gives
\begin{align*}
 \E[\widehat g]
 &=-\frac1\gamma\E_V[\rho'(\Delta(V))m(V)]
 =\bar g+B,\\
 B&=-\frac1\gamma\E_V[(\rho'(\Delta(V))+\gamma)m(V)].
\end{align*}
We now bound the differentiated-link remainder using the additional
third-derivative assumption.

Since $\rho$ is odd and three-times continuously differentiable, $\rho''(0)=0$. Integrating the third-derivative bound twice,
\begin{align*}
    \rho'(z)-\rho'(0)
    &= \int_0^z \rho''(s)\,ds 
    = \int_0^z \left(\rho''(0)+\int_0^s \rho'''(r)\,dr\right)\,ds \\
    &= \int_0^z \int_0^s \rho'''(r)\,dr\,ds 
    = z^2 \int_0^1 \int_0^a \rho'''(bz)\,db\,da \\
    &= z^2 \int_0^1 (1-b)\rho'''(bz)\,db.
\end{align*}
For $|z|\le 2G$ and $b\in[0,1]$, we have $|bz|\le 2G$.
Thus, the bound $|\rho'''(bz)|\le M_3$ gives
\begin{align*}
    |\rho'(z)-\rho'(0)|
    &= z^2 \left|\int_0^1 (1-b)\rho'''(bz)\,db\right|
    \le z^2 \int_0^1 (1-b)|\rho'''(bz)|\,db \\
    &\le M_3 z^2 \int_0^1 (1-b)\,db 
    = \frac{M_3}{2}z^2.
\end{align*}
With $\rho'(0)=-\gamma$, we have \(|\rho'(z)+\gamma| \le \frac{M_3}{2}z^2,\) for \(|z|\le 2G.\)

Since $f$ is differentiable and $G$-Lipschitz, its gradients
at the feasible interior points $x\pm\delta Av$ have norm
at most $G$. Consequently,
\begin{align*}
    \norm{m(v)}_2
    &= \left\| \frac{\nabla f(x+\delta Av)+\nabla f(x-\delta Av)}{2} \right\|_2 
    \le \frac12\left( \norm{\nabla f(x+\delta Av)}_2 + \norm{\nabla f(x-\delta Av)}_2 \right) \\
    &\le G.
\end{align*}
Moreover, by $G$-Lipschitzness,
\begin{align*}
    |\Delta(v)|
    = |f(x+\delta Av)-f(x-\delta Av)| 
    \le G\norm{(x+\delta Av)-(x-\delta Av)}_2 
    = 2\delta G\norm{Av}_2.
\end{align*}
Since $|\Delta(V)|\le 2G$, the previously established bound
$|\rho'(z)+\gamma|\le M_3z^2/2$ applies. Hence,
\begin{align*}
    \norm{B}_2
    &= \frac1\gamma
       \left\|\mathbb E_V
       \left[(\rho'(\Delta(V))+\gamma)m(V)\right]\right\|_2 
    \le \frac1\gamma\mathbb E_V
       \left[
       |\rho'(\Delta(V))+\gamma|\,\norm{m(V)}_2
       \right] \\
    &\le \frac{M_3}{2\gamma}\mathbb E_V
       \left[\Delta(V)^2\norm{m(V)}_2\right] 
    \le \frac{GM_3}{2\gamma}\mathbb E_V[\Delta(V)^2] \\
    &\le \frac{GM_3}{2\gamma}\mathbb E_V
       \left[4\delta^2G^2\norm{AV}_2^2\right] 
    = \frac{2M_3G^3\delta^2}{\gamma}
       \mathbb E_V\left[\norm{AV}_2^2\right].
\end{align*}

For $U\sim\operatorname{Unif}(\mathbb S^{d-1})$, reflecting coordinate $i$ preserves the distribution of $U$ while reversing the sign of $U_iU_j$, thus spherical symmetry gives \(\mathbb E[U_iU_j]=0,\ \forall(i\ne j).\)
Moreover, coordinate symmetry implies that all $\mathbb E[U_i^2]$ are equal. Since $\norm{U}_2^2=1$,
\(
    d\,\mathbb E[U_1^2]
    = \sum_{i=1}^d \mathbb E[U_i^2]
    = \mathbb E[\norm{U}_2^2]
    = 1.
\)
Therefore, \(\mathbb E[UU^\top]=\frac{I}{d}.\)

For $V\sim\operatorname{Unif}(\mathbb B_d)$, write $V=RU$,
where $R\in[0,1]$ and $U$ is uniform on the unit sphere,
independently of $R$. The radial distribution satisfies
\[
    \mathbb P(R\le r)
    = \frac{\operatorname{vol}(r\mathbb B_d)}
           {\operatorname{vol}(\mathbb B_d)}
    = r^d,
    \qquad 0\le r\le1.
\]
Hence its density is $p_R(r)=d\,r^{d-1}$, and
\begin{align*}
    \mathbb E[R^2]
    = \int_0^1 r^2 p_R(r)\,dr 
    = d\int_0^1 r^{d+1}\,dr 
    = \frac{d}{d+2}.
\end{align*}
Using independence,
\begin{align*}
    \mathbb E[VV^\top]
    = \mathbb E[R^2UU^\top] 
    = \mathbb E[R^2]\,\mathbb E[UU^\top] 
    = \frac{d}{d+2}\frac{I}{d} 
    = \frac{I}{d+2}.
\end{align*}

Holding $A$ fixed and using $A=A^\top$, we obtain
\begin{align*}
    \mathbb E_V[\norm{AV}_2^2]
    &= \mathbb E_V[V^\top A^\top AV] 
    = \mathbb E_V[
         \operatorname{tr}(A^\top A VV^\top)] \\
    &= \operatorname{tr}\!\left(
         A^\top A\,\mathbb E_V[VV^\top]\right) 
    = \frac{\operatorname{tr}(A^\top A)}{d+2} 
    = \frac{\operatorname{tr}(A^2)}{d+2}.
\end{align*}
Substituting into the preceding bias estimate yields
\begin{align*}
    \norm{B}_2
    &\le \frac{2M_3G^3\delta^2}{\gamma}
         \mathbb E_V[\norm{AV}_2^2] 
    = \frac{2M_3G^3\delta^2}{\gamma(d+2)}
       \operatorname{tr}(A^2) \\
    &\le \frac{2M_3G^3\delta^2}{\gamma}
       \frac{d}{d+2}\op{A}^2 
    \le \frac{2M_3G^3\delta^2}{\gamma}\op{A}^2,
\end{align*}
where we used
$\operatorname{tr}(A^2)\le d\,\op{A}^2$.

Finally, since $A=H^{-1/2}$, $o^2=1$, and $\norm{U}_2^2=1$, with \(\widehat g=-\frac{d}{2\gamma\delta}oA^{-1}U,\)
\begin{align*}
    \norm{\widehat g}_{H^{-1}}^2
    &= \widehat g^\top H^{-1}\widehat g 
    = \norm{A\widehat g}_2^2 
    = \norm{\frac{d}{2\gamma\delta}oU}^2_2\\
    &= \frac{d^2}{4\gamma^2\delta^2}
       o^2\norm{U}_2^2 
    = \frac{d^2}{4\gamma^2\delta^2}.
\end{align*}
\end{proof}

\subsection{Proof of Lemma~\ref{lem:ellipsoid-feasible-comparison}}
\label{sec:proof-ellipsoid-feasible-comparison}

We compare the center loss with an arbitrary feasible comparator, then add
query smoothing and estimator bias. Identifying the quadratic-surrogate
regret proves the transfer inequality; a trace bound gives its operator-norm form.

\begin{proof}[Proof of Lemma~\ref{lem:ellipsoid-feasible-comparison}]
Fix a round $t$ and a fixed comparator $u\in\mathcal K$, including
boundary points. Condition on $\mathcal F_{t-1}$ and write
$x=w_t$, $A=A_t$, $f=f_t$, and $\widehat g=\widehat g_t$.
The quantities $x$, $A$, $f$, and $u$ are fixed under this
conditioning. All expectations below are conditional on this history;
those involving $\widehat g$ include the query direction and oracle bit.

We first compare the center loss with the comparator loss.
Let $h=\delta AV$ and $y=x+h$, where
$V\sim\operatorname{Unif}(\mathbb B_d)$.
The expectations in this step are over the auxiliary ball point $V$.
By feasibility in Lemma~\ref{lem:smooth-feas}, $y\in\operatorname{int}(\mathcal K)$.

Applying smoothness and strong convexity at the same point $y$ gives respectively
\begin{align*}
    f(x) &\le f(y)+\langle\nabla f(y),x-y\rangle + \frac{L}{2}\norm{x-y}_2^2, \\
    f(u) &\ge f(y)+\langle\nabla f(y),u-y\rangle + \frac{\sigma}{2}\norm{u-y}_2^2.
\end{align*}
Subtracting the second inequality from the first, we obtain
\begin{align*}
    f(x)-f(u)
    &\le
       \langle\nabla f(y),(x-y)-(u-y)\rangle
       +\frac{L}{2}\norm{x-y}_2^2
       -\frac{\sigma}{2}\norm{u-y}_2^2 \\
    &= \langle\nabla f(x+h),x-u\rangle
       +\frac{L}{2}\norm{h}_2^2
       -\frac{\sigma}{2}\norm{u-x-h}_2^2.
\end{align*}

Since $\mathbb E[V]=0$ and $A$ is fixed,
\[
    \mathbb E[h]=\delta A\,\mathbb E[V]=0.
\]
Expanding the squared distance therefore gives
\begin{align*}
    \mathbb E[\norm{u-x-h}_2^2]
    &= \mathbb E\left[
         \norm{u-x}_2^2
         -2\langle u-x,h\rangle
         +\norm{h}_2^2
       \right] \\
    &= \norm{u-x}_2^2
       -2\langle u-x,\mathbb E[h]\rangle
       +\mathbb E[\norm{h}_2^2] \\
    &= \norm{u-x}_2^2+\mathbb E[\norm{h}_2^2].
\end{align*}
Moreover, symmetry of $V$ implies
\begin{align*}
    \mathbb E[\nabla f(x+h)]
    &= \frac12\mathbb E\left[
         \nabla f(x+\delta AV)+\nabla f(x-\delta AV)
       \right] \\
    &= \bar g.
\end{align*}

Taking expectations in the comparison inequality yields
\begin{align*}
    f(x)-f(u)
    &\le
       \langle\mathbb E[\nabla f(x+h)],x-u\rangle
       +\frac{L}{2}\mathbb E[\norm{h}_2^2]
       -\frac{\sigma}{2}
          \mathbb E[\norm{u-x-h}_2^2] \\
    &= \langle\bar g,x-u\rangle
       -\frac{\sigma}{2}\norm{x-u}_2^2
       +\frac{L-\sigma}{2}\mathbb E[\norm{h}_2^2].
\end{align*}
Using $h=\delta AV$ and $\mathbb E[VV^\top]=I/(d+2)$, we have
\(\mathbb E[\norm{h}_2^2]
    =\delta^2\mathbb E[\norm{AV}_2^2]
    =\frac{\delta^2\operatorname{tr}(A^2)}{d+2}.\)
Thus,
\begin{equation}\label{eq:localcomparison}
    f(x)-f(u)
    \le \langle\bar g,x-u\rangle
    -\frac{\sigma}{2}\norm{x-u}_2^2
    +\frac{L-\sigma}{2}\frac{\delta^2\operatorname{tr}(A^2)}{d+2}.
\end{equation}
This intermediate comparison requires no evaluation of $f$ at $u+h$.

Next, define the expected average query loss by \(Q(x)=\mathbb E_U\left[\frac{f(x+\delta AU)+f(x-\delta AU)}{2}\right],\) where \(U\sim\operatorname{Unif}(\mathbb S^{d-1}).\)
By $L$-smoothness at $x$,
\begin{align*}
    f(x+\delta AU)
    &\le f(x)
       +\delta\langle\nabla f(x),AU\rangle
       +\frac{L\delta^2}{2}\norm{AU}_2^2, \\
    f(x-\delta AU)
    &\le f(x)
       -\delta\langle\nabla f(x),AU\rangle
       +\frac{L\delta^2}{2}\norm{AU}_2^2.
\end{align*}
Averaging these two inequalities cancels the linear terms:
\[
    \frac{f(x+\delta AU)+f(x-\delta AU)}{2}
    \le f(x)+\frac{L\delta^2}{2}\norm{AU}_2^2.
\]
Taking expectations and using
$\mathbb E[UU^\top]=I/d$ and $A=A^\top$, we obtain
\begin{align*}
    Q(x)-f(x)
    \le \frac{L\delta^2}{2}\mathbb E_U[\norm{AU}_2^2]
    = \frac{L\delta^2}{2}\operatorname{tr}\!\left(A^\top A\,\mathbb E_U[UU^\top] \right) 
    = \frac{L\delta^2}{2d}\operatorname{tr}(A^2).
\end{align*}

Adding this inequality to \eqref{eq:localcomparison} yields
\begin{align*}
    Q(x)-f(u)
    &= \bigl(Q(x)-f(x)\bigr)+\bigl(f(x)-f(u)\bigr) \\
    &\le
       \langle\bar g,x-u\rangle
       -\frac{\sigma}{2}\norm{x-u}_2^2 
       +\left(\frac{L}{2d} +\frac{L-\sigma}{2(d+2)}\right)
        \delta^2\operatorname{tr}(A^2).
\end{align*}

By the definition of the bias vector,
$B=\mathbb E[\widehat g]-\bar g$, and hence
\begin{align*}
    \langle\bar g,x-u\rangle
    &= \langle\mathbb E[\widehat g]-B,x-u\rangle \\
    &= \mathbb E[\langle\widehat g,x-u\rangle]
       -\langle B,x-u\rangle.
\end{align*}
Here the expectation involving $\widehat g$ is over both
the sampled direction and the oracle bit.

Since $x,u\in\mathcal K$ and
$\operatorname{diam}(\mathcal K)\le D$,
Cauchy--Schwarz and Lemma~\ref{lem:smooth-bias} give
\begin{align*}
    -\langle B,x-u\rangle
    &\le |\langle B,x-u\rangle| \\
    &\le \norm{B}_2\norm{x-u}_2 \\
    &\le D\norm{B}_2 \\
    &\le
       \frac{2M_3G^3D\delta^2}{\gamma(d+2)}
       \operatorname{tr}(A^2).
\end{align*}
Substituting these bounds into the preceding comparison, we conclude that
\begin{align*}
    Q(x)-f(u)
    &\le
       \mathbb E[\langle\widehat g,x-u\rangle]
       -\frac{\sigma}{2}\norm{x-u}_2^2 \\
    &\quad
       +\left(
           \frac{L}{2d}
           +\frac{L-\sigma}{2(d+2)}
           +\frac{2M_3G^3D}{\gamma(d+2)}
        \right)
        \delta^2\operatorname{tr}(A^2) \\
    &=
       \mathbb E[\langle\widehat g,x-u\rangle]
       -\frac{\sigma}{2}\norm{x-u}_2^2
       +c_{d,\sigma}\delta^2\operatorname{tr}(A^2),
\end{align*}
where
\[
    c_{d,\sigma}
    :=
    \frac{L}{2d}
    +\frac{L-\sigma}{2(d+2)}
    +\frac{2M_3G^3D}{\gamma(d+2)}.
\]
Recalling the conditioning and $x=w_t$, we have
\begin{align*}
 Q(x)&=\mathbb E_t\!\left[\frac{f_t(w_t^+)+f_t(w_t^-)}2\right],\\
 q_t(w_t)-q_t(u)
 &=\langle\widehat g_t,w_t-u\rangle
       -\frac\sigma2\|w_t-u\|_2^2.
\end{align*}
Consequently, the preceding bound gives
\begin{equation}\label{eq:querycomparison}
 \mathbb E_t\!\left[\frac{f_t(w_t^+)+f_t(w_t^-)}2-f_t(u)\right]
 \le\mathbb E_t[q_t(w_t)-q_t(u)]
       +c_{d,\sigma}\delta^2\operatorname{tr}(A_t^2),
\end{equation}
which is~\eqref{eq:ell-section-comparison} in Lemma~\ref{lem:ellipsoid-feasible-comparison}.

It remains to prove the lemma's operator-norm bound.
Since $A$ is symmetric positive definite, let
$\lambda_1,\ldots,\lambda_d>0$ denote its eigenvalues.
Then
\[
    \operatorname{tr}(A^2)=\sum_{i=1}^d\lambda_i^2,
    \qquad
    \op{A}=\max_{1\le i\le d}\lambda_i.
\]
Therefore, \(\operatorname{tr}(A^2)
    = \sum_{i=1}^d\lambda_i^2 
    \le \sum_{i=1}^d\op{A}^2 
    = d\,\op{A}^2.\)
Moreover, by the definition of $c_{d,\sigma}$,
\begin{align*}
    d\,c_{d,\sigma}
    &= \frac{L}{2}
       +\frac{d}{d+2}\frac{L-\sigma}{2}
       +\frac{d}{d+2}\frac{2M_3G^3D}{\gamma} 
    \le \frac{L}{2}
       +\frac{L-\sigma}{2}
       +\frac{2M_3G^3D}{\gamma} \\
    &= L-\frac{\sigma}{2}
       +\frac{2M_3G^3D}{\gamma} 
    \le L+\frac{2M_3G^3D}{\gamma}
       =:C,
\end{align*}
where we used $d/(d+2)\le1$ and $0\le\sigma\le L$.

Combining these inequalities, and noting that
$c_{d,\sigma}\ge0$, gives
\[
    c_{d,\sigma}\delta^2\operatorname{tr}(A^2)
    \le d\,c_{d,\sigma}\delta^2\op{A}^2
    \le C\delta^2\op{A}^2.
\]
Dividing by $\delta^2>0$ and recalling $A=A_t$ proves the remaining assertion of Lemma~\ref{lem:ellipsoid-feasible-comparison}.
\end{proof}

\subsection{Pathwise FTRL regret and curvature}
\label{sec:ellipsoid-ftrl-proof}

Lemma~\ref{lem:smooth-ftrl} controls the realized quadratic-surrogate regret. Its proof identifies $H_t$ as the updated objective Hessian, bounds the FTRL movement by self-concordance, and telescopes the surrogate losses.
\begin{lemma}[Pathwise surrogate regret]\label{lem:smooth-ftrl}
For every realized estimator sequence from Algorithm~\ref{alg:dueling_ellipsoidal} and every $u\in\Int\K$, given \(\frac{\eta d}{2\gamma \delta} \leq \frac 12\),
\begin{equation}\label{eq:ftrlbound}
 \sum_{t=1}^T\left[\ip{\widehat g_t}{w_t-u}-\frac\sigma2\norm{w_t-u}^2\right]
 \le\frac{\Phi(u)-\Phi(w_1)}\eta
     +2\eta\sum_{t=1}^T\norm{\widehat g_t}_{H_t^{-1}}^2.
\end{equation}
\end{lemma}
\begin{proof}
This is the application of \cite[Lemma 11]{hazan2014bandit} to the quadratic surrogate losses $q_t$, after subtracting $\Phi(w_1)$ from the regularizer. We give a direct proof to expose the indexing, curvature and step-size requirements, and to cover $\sigma=0$ explicitly.

Recall that the surrogate loss at round $s$ is \(q_s(w)=\langle\widehat g_s,w\rangle + \frac{\sigma}{2}\norm{w-w_s}_2^2.\) When differentiating with respect to $w$, the previously generated vectors $w_s$ and $\widehat g_s$ are held fixed, thus
\[
    \nabla q_s(w)=\widehat g_s+\sigma(w-w_s),
    \qquad
    \nabla^2q_s(w)=\sigma I.
\]
Define \(F_0(w):=\Phi(w),\) and \(F_t(w):=\Phi(w)+\eta\sum_{s=1}^t q_s(w)\) for \(t\ge 1\). By the FTRL update, \(w_t =\arg\min_{w\in\operatorname{int}(\mathcal K)} F_{t-1}(w),\) which also holds for $t=1$, because the initialization is \(w_1=\arg\min_{w\in\operatorname{int}(\mathcal K)}F_0(w).\) Since $w_t$ is an interior minimizer of the differentiable function $F_{t-1}$, its first-order optimality condition is \(\nabla F_{t-1}(w_t) =0.\) Using $F_t=F_{t-1}+\eta q_t$, we therefore obtain
\begin{align*}
    \nabla F_t(w_t)
    =\nabla F_{t-1}(w_t)+\eta\nabla q_t(w_t) 
    =0+\eta\left(
        \widehat g_t+\sigma(w_t-w_t)
       \right) 
    =\eta\widehat g_t.
\end{align*}
For the Hessian, every surrogate contributes $\sigma I$,
including the current surrogate $q_t$. Hence,
\begin{align*}
    \nabla^2F_t(w_t)
    =\nabla^2\Phi(w_t)
      +\eta\sum_{s=1}^t\nabla^2q_s(w_t) 
    =\nabla^2\Phi(w_t)
      +\eta\sum_{s=1}^t\sigma I 
    =\nabla^2\Phi(w_t)+\eta\sigma tI.
\end{align*}
Thus,
\begin{equation}\label{eq:posthessian}
    \nabla F_t(w_t)=\eta\widehat g_t,
    \qquad
    \nabla^2F_t(w_t)
    =\nabla^2\Phi(w_t)+\eta\sigma tI
    =H_t.
\end{equation}
Although $w_t$ minimizes $F_{t-1}$, the matrix $H_t$ is the Hessian of the updated objective $F_t$, evaluated at $w_t$. Equivalently, \(H_t=\nabla^2F_{t-1}(w_t)+\eta\sigma I.\) The additional curvature $\eta\sigma I$ comes from the quadratic part of the current surrogate, whose gradient vanishes at its center but whose Hessian does not. This explains the index $t$, rather than $t-1$, in the sampling matrix $A_t=H_t^{-1/2}$. Moreover, $H_t$ is available before observing the current feedback, so the unknown vector $\widehat g_t$ appears only in a linear term and therefore does not affect the Hessian.

Write \(F_t(w)=\Phi(w)+P_t(w)\) where \(P_t(w):=\eta\sum_{s=1}^t q_s(w).\) Since each $q_s$ is a quadratic function with Hessian $\sigma I$, we have \(\nabla^2P_t(w)=\eta\sigma tI\succeq0,\) and \(D^3P_t(w)[v,v,v]=0.\) Thus for every $z\in\operatorname{int}(\mathcal K)$ and $v\in\mathbb R^d$,
\begin{align*}
    |D^3F_t(z)[v,v,v]|
    &= |D^3\Phi(z)[v,v,v]| 
    \le 2\bigl(v^\top\nabla^2\Phi(z)v\bigr)^{3/2} 
    \le 2\bigl(v^\top\nabla^2F_t(z)v\bigr)^{3/2},
\end{align*}
which concludes that $F_t$ satisfies the standard self-concordance inequality. Moreover, $\nabla^2F_t(z)\succ0$ because $\nabla^2\Phi(z)\succ0$. Let \(h:=w_{t+1}-w_t\), \( r:=\norm{h}_{H_t}, \) and \(\lambda:=\eta\norm{\widehat g_t}_{H_t^{-1}} \le\frac12.\) Recall that \( H_t=\nabla^2F_t(w_t),\) and \(\nabla F_t(w_t)=\eta\widehat g_t.\) If $h=0$, the displacement and stability bounds \(r\le\frac{\lambda}{1-\lambda}\le2\lambda\) hold immediately. 

If $h\ne0$, then $r>0$. Both iterates \(w_t, w_{t+1}\) lie in $\operatorname{int}(\mathcal K)$, and this set is convex. Therefore, the entire segment $w_t+sh$, $s\in[0,1]$, lies in the interior. Define its directional curvature by \(a(s):=h^\top\nabla^2F_t(w_t+sh)h.\) Positive definiteness gives $a(s)>0$, and \(a(0)=h^\top H_t h=r^2.\) Differentiating along the segment gives \(a'(s)=D^3F_t(w_t+sh)[h,h,h].\) By self-concordance,
\[
|a'(s)|\le2a(s)^{3/2}.
\]

To integrate this inequality, define \(b(s):=a(s)^{-1/2}.\) Then \(b'(s)=-\frac12a(s)^{-3/2}a'(s), \) and \(|b'(s)| \le\frac{|a'(s)|}{2a(s)^{3/2}} \le1.\) Hence, for $s\in[0,1]$, \(a(s)^{-1/2} =b(s)  =b(0)+\int_0^s b'(\tau)\,d\tau  \le b(0)+s  =\frac1r+s.\) Since both sides are positive, taking reciprocal squares yields
\[
a(s) \ge\left(\frac1r+s\right)^{-2} =\frac{r^2}{(1+sr)^2}.
\]

Next, let \(\psi(s):=F_t(w_t+sh).\) Its first two derivatives are \(\psi'(s)=\langle\nabla F_t(w_t+sh),h\rangle\) and \(\psi''(s)=a(s)\). Since $w_{t+1}$ is an interior minimizer of $F_t$, first-order condition dictates that $\nabla F_t(w_{t+1})=0$, and therefore $\psi'(1)=0$. The fundamental theorem of calculus gives:
\begin{align*}
    -\langle\nabla F_t(w_t),h\rangle
    &=\psi'(1)-\psi'(0) 
    =\int_0^1 a(s)\,ds \ge\int_0^1\frac{r^2}{(1+sr)^2}\,ds \\
    &=\left[-\frac{r}{1+sr}\right]_{s=0}^{s=1} 
    =r-\frac{r}{1+r} 
    =\frac{r^2}{1+r}.
\end{align*}
On the other hand, Cauchy-Schwarz in the $H_t$ and $H_t^{-1}$ norms gives
\begin{align*}
    -\langle\nabla F_t(w_t),h\rangle
    &\le |\langle\nabla F_t(w_t),h\rangle| 
    \le
       \norm{\nabla F_t(w_t)}_{H_t^{-1}}
       \norm{h}_{H_t} 
    =\eta\norm{\widehat g_t}_{H_t^{-1}}r 
    =\lambda r.
\end{align*}
Combining the lower and upper bounds, we obtain \(\frac{r^2}{1+r}\le\lambda r.\) Since $r>0$, division by $r$ and rearrangement give \((1-\lambda)r\le\lambda.\) Using $\lambda\le1/2$, it follows that
\[
    r\le\frac{\lambda}{1-\lambda}\le2\lambda,
\]
which also holds when $h=0$. By the definition of $q_t$ and $w_{t+1}=w_t+h$,
\begin{align}\label{eq:stability-bound}
    q_t(w_t)-q_t(w_{t+1})
    &=
       \langle\widehat g_t,w_t\rangle
       -\left(
          \langle\widehat g_t,w_t+h\rangle
          +\frac{\sigma}{2}\norm{h}_2^2
        \right) \nonumber \\
    &= -\langle\widehat g_t,h\rangle
       -\frac{\sigma}{2}\norm{h}_2^2 
    \le |\langle\widehat g_t,h\rangle| 
    \le \norm{\widehat g_t}_{H_t^{-1}} \norm{h}_{H_t} \nonumber \\
    &= \norm{\widehat g_t}_{H_t^{-1}}r 
    \le 2\lambda\norm{\widehat g_t}_{H_t^{-1}} 
    = 2\eta\norm{\widehat g_t}_{H_t^{-1}}^2.
\end{align}

Define \(J_t(w):=\frac{\Phi(w)}{\eta}+\sum_{s=1}^t q_s(w)\) and \(S_t:=\min_{w\in\operatorname{int}\mathcal K}J_t(w),\) for \(t=0,\ldots,T\), where the sum is empty when $t=0$. Since $J_t=F_t/\eta$ and $\eta>0$, the update rule implies that $w_{t+1}$ minimizes $J_t$. Therefore,
\begin{align*}
S_t =J_t(w_{t+1}) =J_{t-1}(w_{t+1})+q_t(w_{t+1}) \ge S_{t-1}+q_t(w_{t+1}),
\end{align*}
where the inequality follows from
$S_{t-1}=\min_{w\in\operatorname{int}\mathcal K}J_{t-1}(w)$.
Rearranging and summing over $t$ gives
\[
\sum_{t=1}^T q_t(w_{t+1})
\le
\sum_{t=1}^T(S_t-S_{t-1})
=S_T-S_0.
\]
Evaluating $J_T$ at the comparator $u$ yields \[S_T\le J_T(u)=\frac{\Phi(u)}{\eta}+\sum_{t=1}^Tq_t(u).\]
Moreover, since $w_1$ minimizes $\Phi$,
\[S_0
=\min_{w\in\operatorname{int}\mathcal K}\frac{\Phi(w)}{\eta}
=\frac{\Phi(w_1)}{\eta}.\]
Combining these relations, we obtain
\begin{align}
\sum_{t=1}^T \bigl[q_t(w_{t+1})-q_t(u)\bigr]
\le \frac{\Phi(u)-\Phi(w_1)}{\eta}.
\end{align}
Add and subtract $q_t(w_{t+1})$ and apply the stability bound \eqref{eq:stability-bound}:
\begin{align*}
\sum_{t=1}^T\bigl[q_t(w_t)-q_t(u)\bigr]
&=
\sum_{t=1}^T\bigl[q_t(w_t)-q_t(w_{t+1})\bigr]
+ \sum_{t=1}^T\bigl[q_t(w_{t+1})-q_t(u)\bigr]\\
&\le 2\eta\sum_{t=1}^T\|\widehat g_t\|_{H_t^{-1}}^2
+\frac{\Phi(u)-\Phi(w_1)}{\eta}.
\end{align*}
Using \(q_t(w)=\langle\widehat g_t,w\rangle +\frac{\sigma}{2}\|w-w_t\|_2^2,\) we have
\begin{align*}
q_t(w_t)-q_t(u)
&=
\langle\widehat g_t,w_t\rangle
-\left(
\langle\widehat g_t,u\rangle
+\frac{\sigma}{2}\|u-w_t\|_2^2
\right)\\
&=
\langle\widehat g_t,w_t-u\rangle
-\frac{\sigma}{2}\|w_t-u\|_2^2.
\end{align*}
Substituting \(q_t(w_t)-q_t(u)\), we have
\[
\sum_{t=1}^T
\left[
\langle\widehat g_t,w_t-u\rangle
-\frac{\sigma}{2}\|w_t-u\|_2^2
\right]
\le
\frac{\Phi(u)-\Phi(w_1)}{\eta}
+2\eta\sum_{t=1}^T\|\widehat g_t\|_{H_t^{-1}}^2.
\]
Every step holds for each realization satisfying the stated step-size condition. Thus the bound is pathwise and requires no unbiasedness assumption on the estimators.
\end{proof}

\subsection{Proof of Theorem~\ref{thm:universal_ellipsoidal}}
\label{sec:proof-ellipsoidal-master}

We combine Lemmas~\ref{lem:ellipsoid-feasible-comparison} and~\ref{lem:smooth-ftrl} for an interior comparator, then use barrier growth to approximate any boundary comparator. The approximation cost is charged in the original losses.

\begin{proof}
\textbf{Interior comparator bound.} Fix a deterministic comparator $u\in\operatorname{int}(\mathcal K)$. Use the pre-query history $\mathcal F_{t-1}$ from Section~\ref{sec:dtfo} and write \(\mathbb E_t[\cdot]:=\mathbb E[\cdot\mid\mathcal F_{t-1}].\) Under the oblivious-adversary assumption, the loss functions are fixed. Moreover, $w_t$, $A_t$, and $H_t$ are $\mathcal F_{t-1}$-measurable. Define the realized average query loss by
\[
    L_t^{\mathrm q}
    :=\frac{f_t(w_t^+)+f_t(w_t^-)}{2}.
\]
Applying Lemma~\ref{lem:ellipsoid-feasible-comparison}, proved in Appendix~\ref{sec:proof-ellipsoid-feasible-comparison}, and expanding
$q_t(w_t)-q_t(u)$ gives
\begin{align*}
    \mathbb E_t[L_t^{\mathrm q}]-f_t(u)
    &\le
    \left\langle
        \mathbb E_t[\widehat g_t],w_t-u
    \right\rangle
    -\frac{\sigma}{2}\norm{w_t-u}_2^2
    +c_{d,\sigma}\delta^2\operatorname{tr}(A_t^2).
\end{align*}
Since $w_t-u$ is $\mathcal F_{t-1}$-measurable, the tower property gives
\begin{align*}
    \mathbb E\left[
        \left\langle
            \mathbb E_t[\widehat g_t],w_t-u
        \right\rangle
    \right]
    =
    \mathbb E\left[
        \mathbb E_t\left[
            \langle\widehat g_t,w_t-u\rangle
        \right]
    \right] 
    =
    \mathbb E\left[
        \langle\widehat g_t,w_t-u\rangle
    \right].
\end{align*}
Taking expectations in the conditional comparison inequality therefore yields
\begin{align*}
    \mathbb E[L_t^{\mathrm q}]-f_t(u)
    \le
    \mathbb E\left[
        \langle\widehat g_t,w_t-u\rangle
        -\frac{\sigma}{2}\norm{w_t-u}_2^2
    \right] 
    +c_{d,\sigma}\delta^2
        \mathbb E[\operatorname{tr}(A_t^2)].
\end{align*}
Summing over $t=1,\ldots,T$, we obtain
\begin{align*}
    R_T(u)
    =
    \sum_{t=1}^T
    \left(\mathbb E[L_t^{\mathrm q}]-f_t(u)\right) 
    \le
    \mathbb E\left[
        \sum_{t=1}^T
        \left(
            \langle\widehat g_t,w_t-u\rangle
            -\frac{\sigma}{2}\norm{w_t-u}_2^2
        \right)
    \right] 
    +c_{d,\sigma}\delta^2
        \sum_{t=1}^T
        \mathbb E[\operatorname{tr}(A_t^2)].
\end{align*}

For the quadratic surrogate losses \( q_t(w) =\langle\widehat g_t,w\rangle +\frac{\sigma}{2}\norm{w-w_t}_2^2,\) we have
\[
    q_t(w_t)-q_t(u)
    =
    \langle\widehat g_t,w_t-u\rangle
    -\frac{\sigma}{2}\norm{w_t-u}_2^2.
\]
Thus, Lemma~\ref{lem:smooth-ftrl} applies pathwise and gives
\begin{align*}
    R_T(u)
    \le
    \frac{\Phi(u)-\Phi(w_1)}{\eta}
    +2\eta\sum_{t=1}^T
        \mathbb E\left[
            \norm{\widehat g_t}_{H_t^{-1}}^2
        \right]
    +c_{d,\sigma}\delta^2
        \sum_{t=1}^T
        \mathbb E[\operatorname{tr}(A_t^2)].
\end{align*}
Finally, Lemma~\ref{lem:smooth-bias} establishes the
deterministic identity
\[
    \norm{\widehat g_t}_{H_t^{-1}}^2
    =\frac{d^2}{4\gamma^2\delta^2}.
\]
Substituting this identity yields
\begin{equation}\label{eq:masterinterior}
 R_T(u)\le\frac{\Phi(u)-\Phi(w_1)}\eta
       +\frac{\eta d^2T}{2\gamma^2\delta^2}
       +c_{d,\sigma}\delta^2\sum_{t=1}^T\E\tr(A_t^2).
\end{equation}
This is the required interior-comparator bound.

\noindent \textbf{Best-fixed-in-hindsight comparator regret.} Fix any $u\in\K$ and $0<\alpha<1$, and define \(u_\alpha:=(1-\alpha)u+\alpha w_1.\) We first verify that $u_\alpha\in\operatorname{int}\mathcal K$. Since $w_1\in\operatorname{int}\mathcal K$, there exists $\varepsilon>0$ such that the open Euclidean ball $B(w_1,\varepsilon)$ is contained in $\mathcal K$.
By convexity, \(B(u_\alpha,\alpha\varepsilon) =(1-\alpha)u+\alpha B(w_1,\varepsilon) \subseteq\mathcal K, \) and thus the interior-comparator bound \eqref{eq:masterinterior} applies to $u_\alpha$.

We prove the required barrier-growth bound directly from
\[
 |D\Phi(x)[v]|^2\le\nu D^2\Phi(x)[v,v],
\]
which is the gradient inequality in the definition of a $\nu$-self-concordant barrier. Set $h=u-w_1$ and define
\[
\varphi(s):=\Phi(w_1+sh),
\qquad
b(s):=\exp\left(-\frac{\varphi(s)}{\nu}\right),
\qquad 0\le s<1.
\]
The segment points $w_1+sh$ lie in $\operatorname{int}\mathcal K$ for $s<1$. By the chain rule and the barrier's gradient inequality,
\[
\varphi'(s)=D\Phi(w_1+sh)[h],
\qquad
\varphi''(s)=D^2\Phi(w_1+sh)[h,h],
\qquad
(\varphi'(s))^2\le\nu\varphi''(s).
\]
Consequently,
\[
b''(s)
=
b(s)\left(
\frac{(\varphi'(s))^2}{\nu^2}
-\frac{\varphi''(s)}{\nu}
\right)
\le0.
\]
Thus $b$ is positive and concave. For $0<s<r<1$,
\[
b(s)
\ge
\left(1-\frac{s}{r}\right)b(0)+\frac{s}{r}b(r)
\ge
\left(1-\frac{s}{r}\right)b(0).
\]
Letting $r\uparrow1$ gives $b(s)\ge(1-s)b(0)$. Taking logarithms and rearranging yields
\[
\varphi(s)-\varphi(0)
\le\nu\log\frac{1}{1-s}.
\]
The choice $s=1-\alpha$ proves
$\Phi(u_\alpha)-\Phi(w_1)\le\nu\log(1/\alpha)$.

Next, since $\mathcal K$ has diameter $D$, we have $\|u_\alpha-u\|_2 =\alpha\|w_1-u\|_2 \le\alpha D,$ and the $G$-Lipschitz continuity of each $f_t$ therefore implies
\[
f_t(u_\alpha)-f_t(u)
\le G\|u_\alpha-u\|_2
\le\alpha GD.
\]
For the fixed oblivious loss sequence, the definition of $R_T(u)$ gives the exact identity
\[
R_T(u)
=
R_T(u_\alpha)
+\sum_{t=1}^T
\bigl[f_t(u_\alpha)-f_t(u)\bigr].
\]
Applying \eqref{eq:masterinterior} at $u_\alpha$, we obtain
\begin{align*}
R_T(u)
&\le R_T(u_\alpha)+\alpha GDT\\
&\le
\frac{\Phi(u_\alpha)-\Phi(w_1)}{\eta}
+\frac{\eta d^2T}{2\gamma^2\delta^2}
+c_{d,\sigma}\delta^2
 \sum_{t=1}^T\mathbb E\operatorname{tr}(A_t^2)
+\alpha GDT\\
&\le
\frac{\nu\log(1/\alpha)}{\eta}
+\frac{\eta d^2T}{2\gamma^2\delta^2}
+C\delta^2
 \sum_{t=1}^T\mathbb E\|A_t\|_{\mathrm{op}}^2
+\alpha GDT,
\end{align*}
where \(C=L+\frac{2M_3G^3D}{\gamma},\) and the last inequality uses the pointwise bound \(c_{d,\sigma}\operatorname{tr}(A_t^2)
\le C\|A_t\|_{\mathrm{op}}^2.\)
Since $A_t$ may depend on the random history, its squared operator norm remains inside the expectation.

The resulting upper bound is independent of $u$. For an oblivious loss sequence, a minimizer
\[
u^\star\in
\operatorname*{arg\,min}_{u\in\mathcal K}
\sum_{t=1}^T f_t(u)
\]
exists by compactness and continuity, and can be chosen independently of the algorithm's randomness. Consequently,
\begin{equation}\label{eq:master}
 R_T\le\frac{\nu\log(1/\alpha)}\eta
       +\frac{\eta d^2T}{2\gamma^2\delta^2}
       +C\delta^2\sum_{t=1}^T\E\op{A_t}^2+\alpha GDT.
\end{equation}
The comparator approximation is performed on the original losses through $f_t(u_\alpha)-f_t(u)$ and requires no bound on the Euclidean norms of the linear estimators.
\end{proof}

\subsection{Proof of Corollary~\ref{cor:ellipsoid_smooth}}
\label{sec:proof-ellipsoidal-convex}

We substitute $\|A_t\|_{\mathrm{op}}\le D$ into the master bound and balance learning and exploration errors. The capped parameters remain feasible at finite horizons, and the zero-geometric-error case is treated separately.

\begin{proof}
Let $\sigma=0$ and $C_{\mathrm{geom}}=CD^2>0$. Write $B_T=\nu\log(T+1)$ and $V_0=d^2/(2\gamma^2)$. Set $\alpha=1/(T+1)$ and use $\op{A_t}^2\le D^2$ from Lemma~\ref{lem:smooth-feas} in~\eqref{eq:master} to obtain
\begin{equation}\label{eq:convexfinite}
 R_T\le B_T/\eta+V_0\eta T/\delta^2+C_{\mathrm{geom}}\delta^2T+GD.
\end{equation}
For fixed $\delta$, balancing the first two terms gives $\eta=\delta\sqrt{B_T/(V_0T)}$ and their sum is $2\sqrt{B_TV_0T}/\delta$. Minimizing this plus $C_{\mathrm{geom}}\delta^2T$ gives the uncapped choices
\begin{equation}\label{eq:convextuning}
 \delta_* =\left(\frac{B_TV_0}{C_{\mathrm{geom}}^2T}\right)^{1/6},
 \qquad \eta_*=\delta_*\sqrt{\frac{B_T}{V_0T}}.
\end{equation}
At these choices, the first two terms together equal twice the third, so
\begin{equation}\label{eq:convexrate}
 R_T\le3(B_TV_0C_{\mathrm{geom}})^{1/3}T^{2/3}+GD
       =O\!\left(T^{2/3}(\log(T+1))^{1/3}\right).
\end{equation}
The conditions needed are $\delta_*\le\delta_0$ and $B_T\le T/2$; the latter gives $\eta_*d/(2\gamma\delta_*)\le1/2$. Both hold eventually for fixed problem constants.

For every finite $T$, a valid choice is $\delta=\min\{\delta_0,\delta_*\}$ and $\eta=\min\{\delta\sqrt{B_T/(V_0T)},\gamma\delta/d\}$; use~\eqref{eq:convexfinite} directly when a cap is active. If $C_{\mathrm{geom}}=0$, take a fixed feasible $\delta$ and the capped learning rate $\eta=\min\{\delta\sqrt{B_T/(V_0T)},\gamma\delta/d\}$. Once the cap is inactive, balancing just the first two terms gives $O(\sqrt{T\log(T+1)})$. This includes affine losses with a locally linear oracle.
\end{proof}

\subsection{Proof of Corollary~\ref{cor:ellipsoid_strong_smooth}}
\label{sec:proof-ellipsoidal-strong}

We use $\|A_t\|_{\mathrm{op}}^2\le1/(\eta\sigma t)$ to sum the geometric error harmonically, then optimize the learning rate with a fixed feasible radius.

\begin{proof}
Let $\sigma>0$, choose any $0<\delta\le\delta_0$ fixed independently of $T$, and set $\alpha=1/(T+1)$. For
$\mathsf H_T=\sum_{t=1}^T1/t$, define
\[
 V_\delta=\frac{d^2}{2\gamma^2\delta^2},\qquad
 K_T=\nu\log(T+1)+\frac{C\delta^2}{\sigma}\mathsf H_T.
\]
Use $\op{A_t}^2\le1/(\eta\sigma t)$ from Lemma~\ref{lem:smooth-feas} in~\eqref{eq:master} to obtain
\begin{equation}\label{eq:strongfinite}
 R_T\le V_\delta\eta T+K_T/\eta+GD.
\end{equation}
Balancing the first two terms gives $\eta=\sqrt{K_T/(V_\delta T)}$; imposing the cap yields $\eta=\min\{\sqrt{K_T/(V_\delta T)},\gamma\delta/d\}$. Once the cap is inactive,
\begin{equation}\label{eq:strongrate}
 R_T\le2\sqrt{V_\delta T K_T}+GD
       =O\!\left(\sqrt{T\log(T+1)}\right),
\end{equation}
where $\mathsf H_T\le1+\log T$. The cap is inactive whenever $K_T\le T/2$, which eventually holds for fixed problem constants.
\end{proof}

\end{document}